\documentclass[final]{IEEEtran}
\usepackage[noadjust]{cite}

\usepackage[dvips]{graphicx}
\usepackage{multirow,multicol}
\usepackage[table]{xcolor}
\usepackage{ctable}

\usepackage[tbtags]{amsmath}
\usepackage{amsbsy}
\usepackage{amssymb}
\usepackage{amsfonts}
\usepackage{dsfont}

\usepackage{graphicx}
\usepackage{subfig}
\usepackage{url}
\usepackage{textcomp}
\usepackage{algorithmic}
\usepackage{algorithm}
\usepackage{balance}
\usepackage{rotating}
\usepackage{caption}
\usepackage{fancyvrb}
\usepackage{latexsym}
\usepackage{verbatim}

\usepackage{etex}
\usepackage{tikz}
\usepackage{pgfplots}

\newtheorem{proposition}{Proposition}

\newtheorem{remark}{Remark}

{\begin{list}               
    {$\bullet$ \hfill}{
        \setlength{\leftmargin}{\parindent}
        \setlength{\parsep}{0.04\baselineskip}
        \setlength{\itemsep}{0.5\parsep}
        \setlength{\labelwidth}{\leftmargin}
        \setlength{\labelsep}{0em}}
    }
{\end{list}}

\providecommand{\eref}[1]{\eqref{#1}}  
\providecommand{\cref}[1]{Chapter~\ref{#1}}

\providecommand{\fref}[1]{Figure~\ref{#1}}

\providecommand{\R}{\ensuremath{\mathbb{R}}}
\providecommand{\C}{\ensuremath{\mathbb{C}}}
\providecommand{\E}{\ensuremath{\mathbb{E}}}

\providecommand{\bydef}{\overset{\text{def}}{=}}

\renewcommand{\vec}[1]{\ensuremath{\boldsymbol{#1}}}
\providecommand{\mat}[1]{\ensuremath{\boldsymbol{#1}}}

\providecommand{\calN}{\mathcal{N}}

\providecommand{\mI}{\mat{I}}

\providecommand{\mP}{\mat{P}}

\providecommand{\mS}{\mat{S}}

\providecommand{\vb}{\vec{b}}

\providecommand{\vp}{\vec{p}}
\providecommand{\vq}{\vec{q}}

\providecommand{\vv}{\vec{v}}

\providecommand{\vx}{\vec{x}}
\providecommand{\vy}{\vec{y}}

\providecommand{\mPsi}{\mat{\Psi}}
\providecommand{\mSigma}{\mat{\Sigma}}
\providecommand{\mTheta}{\mat{\Theta}}

\providecommand{\vepsilon}{\vec{\epsilon}}
\providecommand{\vvarepsilon}{\vec{\varepsilon}}

\providecommand{\vmu}{\vec{\mu}}

\providecommand{\vvartheta}{\vec{\vartheta}}

\providecommand{\vxbar}{\boldsymbol{\overline{x}}}

\providecommand{\vzero}{\vec{0}}

\newcommand{\defequal}{\mathop{\overset{\mbox{\tiny{def}}}{=}}}
\newcommand{\subjectto}{\mathop{\mathrm{subject\, to}}}
\newcommand{\argmin}[1]{\mathop{\underset{#1}{\mathrm{arg\,min}}}}
\newcommand{\argmax}[1]{\mathop{\underset{#1}{\mbox{argmax}}}}

\newcommand{\maximize}[1]{\mathop{\underset{#1}{\mathrm{maximize}}}}

\graphicspath{{pix/}}
\usepackage[hmargin={1.8cm}, vmargin={2cm}]{geometry}

\title{Adaptive Image Denoising by Mixture Adaptation}
\author{Enming Luo,~\IEEEmembership{Student Member,~IEEE}, Stanley~H.~Chan,~\IEEEmembership{Member,~IEEE}, and Truong~Q.~Nguyen,~\IEEEmembership{Fellow,~IEEE}
\thanks{E. Luo and T. Nguyen are with the Department of Electrical and Computer Engineering, University of California at San Diego, La Jolla, CA 92093, USA. (emails: eluo@ucsd.edu and nguyent@ece.ucsd.edu).}
\thanks{S. Chan is with the School of Electrical and Computer Engineering, and the Department of Statistics, Purdue University, West Lafayette, IN 47907, USA. (email: stanleychan@purdue.edu).}
\thanks{This work was supported by the National Science Foundation under grant CCF-1065305 and by DARPA under contract W911NF-11-C-0210. Preliminary material in this paper was presented at the 3rd IEEE Global Conference on Signal \& Information Processing (GlobalSIP), Orlando, December, 2015.}
\thanks{This paper follows the concept of reproducible research. All the results and examples presented in the paper are reproducible using the code and images available online at http://videoprocessing.ucsd.edu/\texttildelow{}eluo/files/publications/guidedDenoising.zip.}
}

\graphicspath{{pix/}}

\begin{document}
\maketitle

\begin{abstract}
We propose an adaptive learning procedure to learn patch-based image priors for image denoising. The new algorithm, called the Expectation-Maximization (EM) adaptation, takes a generic prior learned from a generic external database and adapts it to the noisy image to generate a specific prior. Different from existing methods that combine internal and external statistics in ad-hoc ways, the proposed algorithm is rigorously derived from a Bayesian hyper-prior perspective. There are two contributions of this paper: First, we provide full derivation of the EM adaptation algorithm and demonstrate methods to improve the computational complexity. Second, in the absence of the latent clean image, we show how EM adaptation can be modified based on pre-filtering. Experimental results show that the proposed adaptation algorithm yields consistently better denoising results than the one without adaptation and is superior to several state-of-the-art algorithms.
\end{abstract}

\begin{keywords}
Image Denoising, Hyper Prior, Conjugate Prior, Gaussian Mixture Models, Expectation-Maximization (EM), Expected Patch Log-Likelihood (EPLL), EM Adapation, BM3D
\end{keywords}

\section{Introduction}
\subsection{Overview}
We consider the classical image denoising problem: Given an additive i.i.d. Gaussian noise model,
\begin{equation}
\vy = \vx + \vvarepsilon,
\label{noisy image model}
\end{equation}
our goal is to find an estimate of $\vx$ from $\vy$, where $\vx \in \R^n$ denotes the (unknown) clean image, $\vvarepsilon \sim
\calN(\vzero, \sigma^2\mI) \in \R^n$ denotes the Gaussian noise vector with zero mean and covariance matrix $\sigma^2\mI$ (where $\mI$ is the identity matrix), and $\vy \in \R^n$ denotes the observed noisy image.

Image denoising is a long-standing problem. Over the past few decades, numerous denoising algorithms have been proposed, ranging from spatial domain methods \cite{Tomasi_Manduchi_1998, Buades_Coll_2005_Journal, Kervrann_Boulanger_2007} to transform domain methods \cite{Dabov_Foi_Katkovnik_2007, Zhang_Dong_Zhang_2010, Yan_Shao_Liu_2013}, and from local filtering \cite{Takeda_Farsiu_Milanfar_2007,Luo_Pan_Nguyen_2012,
Luo_Chan_Nguyen_2014} to global optimization \cite{Talebi_Milanfar_2014, Zoran_Weiss_2011}. In this paper, we focus on the Maximum a Posteriori (MAP) approach \cite{Zoran_Weiss_2011, Elad_Aharon_2006}. MAP is a Bayesian approach which tackles image denoising by maximizing the posterior probability
\begin{equation*}
\argmax{\vx}\,f(\vy|\vx)f(\vx) = \argmin{\vx}\, \left\{\frac{1}{2\sigma^2}\|\vy-\vx\|^2 - \log f(\vx)\right\}.
\end{equation*}
Here, the first term is a quadratic function due to the Gaussian noise model. The second term is the negative log of the image prior. The benefit of using the MAP framework is that it allows us to explicitly formulate our prior knowledge about the image via the distribution $f(\vx)$.

The success of an MAP optimization depends vitally on the modeling capability of the prior $f(\vx)$ \cite{Yu_Sapiro_Mallat_2012,Roth_Black_2009,
Roth_Black_2005}. However, seeking $f(\vx)$ for the whole image $\vx$ is practically impossible because of the high dimensionality. To alleviate the problem, we adopt the common wisdom to approximate $f(\vx)$ using a collection of small patches. Such prior is known as the \emph{patch prior}, which is broadly attributed to Buades \emph{et al.} for the non-local means \cite{Buades_Coll_Morel_2005CVPR}, and to an independent work of Awate and Whitaker presented at the same conference \cite{Awate_Whitaker_2005}. (See \cite{Milanfar_2013a} for additional discussions about patch priors.) Mathematically, by letting $\mP_i \in \R^{d \times n}$ be a patch-extract operator that extracts the $i$-th $d$-dimensional patch from the image $\vx$, a patch prior expresses the negative logarithm of the prior as a sum of the logarithms, leading to
\begin{equation}
\argmin{\vx}\, \left\{\frac{1}{2\sigma^2}\|\vy-\vx\|^2 - \frac{1}{n} \sum_{i=1}^n \log f(\mP_i\vx)\right\}.
\label{MAP formulation for patches}
\end{equation}
The prior thus formed is called the expected patch log likelihood (EPLL) \cite{Zoran_Weiss_2011}.

\subsection{Related Work}
Assuming that $f(\mP_i\vx)$ takes a parametric form for analytic tractability, the question now becomes where to find training samples and how to train the model. There are generally two approaches. The first approach is to learn $f(\mP_i\vx)$ from the single noisy image. We refer to these types of priors as \emph{internal} priors, e.g., \cite{Zontak_Irani_2011}. The second approach is to learn $f(\mP_i\vx)$ from a database of images. We call these types of priors as \emph{external} priors, e.g.,  \cite{Mosseri_Zontak_Irani_2013, Burger_Schuler_Harmeling_2013, Schmidt_Roth_2014,Luo_Chan_Nguyen_2014,Luo_Chan_Nguyen_2015}.

Combining internal and external priors has been an active direction in recent years. Most of these methods are based on a \emph{fusion} approach, which attempts to directly aggregate the results of the internal and the external statistics. For example, Mosseri \emph{et al.} \cite{Mosseri_Zontak_Irani_2013} used a patch signal-to-noise ratio as a metric to decide if a patch should be denoised internally or externally; Burger \emph{et al.} \cite{Burger_Schuler_Harmeling_2013} applied a neural network to weight the internal and external denoising results; Yue \emph{et al.} \cite{Yue_Sun_Yang_Wu_2014} used a frequency domain method to fuse the internal and external denoising results. There are also some works attempting to use external databases as a guide to train internal priors \cite{Tibell_Spies_Borga_2009,Chen_Zhang_Yu_2015}.

When $f(\mP_i\vx)$ is a Gaussian mixture model, there are special treatments to optimize the performance, e.g., a framework proposed by Awate and Whitaker \cite{Awate_Whitaker_2005_v1,Awate_Whitaker_2006,Awate_Whitaker_2007}. In this method, a simplified Gaussian mixture (using the same weights and shared covariances) is learned directly from the noisy data through an empirical Bayes framework. However, the method primarily focuses on MRI data where the noise is Rician. This is different from the i.i.d. Gaussian noise assumption in our problem. The learning procedure is also different from ours as we use an adaptation process to adapt the generic prior to a specific prior. Our proposed method is inspired by the work of Gauvain and Lee \cite{Gauvain_Lee_1994} with a few important modifications.

We should also mention the work of Weissman \emph{et al.} on universal denoising \cite{Weissman_Ordentlich_Seroussi_2005,Sivaramakrishnan_Weissman_2009}. Universal denoisers are a general class of denoising algorithms that do not require explicit knowledge about the prior and are also asymptotically optimal. While not explicitly proven, patch-based denoising methods such as non-local means \cite{Buades_Coll_Morel} and BM3D \cite{Dabov_Foi_Katkovnik_2007} satisfy these properties. For example, the asymptotic optimality of non-local means was empirically verified by Levin \emph{et al.} \cite{Levin_Nadler_2011,Levin_Nadler_Durand_Freeman_2012} with computational improvements by Chan \emph{et al.} \cite{Chan_Zickler_Lu_2014}. However, we shall not discuss universal denoisers in detail as they are beyond the scope of this paper.

\subsection{Contribution and Organization}
Our proposed algorithm is call \emph{EM adaptation}. Like many external methods, we assume that we have an external database of images for training. However, we do not simply compute the statistics of the external database. Instead, we use the external statistics as a ``guide'' for learning the internal statistics. As will be illustrated in the subsequent sections, this can be formally done using a Bayesian framework.

This paper is an extension of our previous work reported in \cite{Chan_Luo_Nguyen_2015}. This paper adds the following two new contributions:
\begin{enumerate}
\item Derivation of the EM adaptation algorithm. We rigorously derive the proposed EM adaptation algorithm from a Bayesian hyper-prior perspective. Our derivation complements the work of Gauvain and Lee \cite{Gauvain_Lee_1994} by providing additional simplifications and justifications to reduce computational complexity. We further provide discussion of the convergence.

\item Handling of noisy data. We provide detailed discussion of how to perform EM adaptation for noisy images. In particular, we demonstrate how to automatically adjust the internal parameters of the algorithm using pre-filtered images.
\end{enumerate}

When this paper was written, we became aware of a very recent work by Lu \emph{et al.} \cite{Lu_Lin_Jin_Yang_Wang_2015}.
Compared to \cite{Lu_Lin_Jin_Yang_Wang_2015}, this paper provides theoretical results that are lacking in \cite{Lu_Lin_Jin_Yang_Wang_2015}. Numerical comparisons can be found in the experiment section.

The rest of the paper is organized as follows: Section II gives a brief review of the Gaussian mixture model. Section III presents the proposed EM adaptation algorithm. Section IV discusses how the EM adaptation algorithm should be modified when the image is noisy. Experimental results are presented in Section V.

\section{Mathematical Preliminaries}

\subsection{GMM and MAP Denoising}
For notational simplicity, we shall denote $\vp_i \bydef \mP_i\vx \in \R^d$ as the $i$-th patch from $\vx$. We say that $\vp_i$ is generated from a Gaussian mixture model (GMM) if
\begin{equation}
f(\vp_i \,|\, \mTheta) = \sum_{k=1}^K \; \pi_k \calN(\vp_i | \vmu_k, \mSigma_k),
\label{GMM distribution}
\end{equation}
where $\sum_{k=1}^K \pi_k = 1$ with $\pi_k$ being the weight of the $k$-th Gaussian component, and
\begin{align}
&\calN(\vp_i | \vmu_k, \mSigma_k) \\
&\bydef \frac{1}{(2\pi)^{d/2}|\mSigma_k|^{1/2}} \mbox{exp} \Big( -\frac{1}{2} (\vp_i - \vmu_k)^T \mSigma_k^{-1} (\vp_i - \vmu_k) \Big) \nonumber
\end{align}
is the $k$-th Gaussian distribution with mean $\vmu_k$ and covariance $\mSigma_k$. We denote $\mTheta \bydef \{(\pi_k, \vmu_k,
\mSigma_k)\}_{k=1}^K$ as the GMM parameter.

With the GMM defined in \eref{GMM distribution}, we can specify the denoising procedure by solving the optimization problem in \eref{MAP formulation for patches}. Here, we follow \cite{Geman_Yang_2001,Krishnan_Fergus_2009} by using the half quadratic splitting strategy. The idea is to replace \eref{MAP formulation for patches} with the following minimization
\begin{align}
\argmin{\vx, \{\vv_i\}_{i = 1}^n}
&\Bigg\{\frac{1}{2\sigma^{2}}\|\vy-\vx\|^2 \notag \\
&+ \frac{1}{n} \sum_{i=1}^n \Big( -\log f(\vv_i) + \frac{\beta}{2}\|\mP_i\vx - \vv_i\|^2 \Big) \Bigg\},
\label{MAP formulation: half quadratic splitting}
\end{align}
where $\{\vv_i\}_{i=1}^n$ are some auxiliary variables and $\beta$ is a penalty parameter. By assuming that $f(\vv_i)$ is dominated by the mode of the Gaussian mixture, the solution to \eref{MAP formulation: half quadratic splitting} is given in the following proposition.

\begin{proposition}
\label{MAP denoising using half quadratic splitting strategy}
Assuming $f(\vv_i)$ is dominated by the $k_i^*$-th components, where
$k_i^* \defequal \argmax{k} \, \pi_k \calN(\vv_i|\vmu_k, \mSigma_k)$, the solution of \eref{MAP formulation: half quadratic splitting} is
\begin{align*}
\vx   &= \Big( n\sigma^{-2}\mI + \beta \sum_{i=1}^n \mP_i^T\mP_i\Big)^{-1} \Big( n\sigma^{-2}\vy + \beta \sum_{i=1}^n\mP_i^T\vv_i \Big),\\
\vv_i &= \left( \beta {\mSigma}_{k_i^*} + \mI \right)^{-1} \left( {\vmu}_{k_i^*} + \beta{\mSigma}_{k_i^*} \mP_i\vx \right).
\end{align*}
\end{proposition}

\begin{proof}
See \cite{Zoran_Weiss_2011}.
\end{proof}

Proposition \ref{MAP denoising using half quadratic splitting strategy} is a general procedure for denoising images using a GMM under the MAP framework. There are, of course, other possible denoising procedures that also use GMM under the MAP framework, e.g., using surrogate methods \cite{Bouman_2015}. However, we will not elaborate on these options. Our focus is on how to obtain the GMM.

\subsection{EM Algorithm}
The GMM parameter $\mTheta = \{(\pi_k, \vmu_k, \mSigma_k)\}_{k=1}^K$ is typically learned using the Expectation-Maximization (EM) algorithm. EM is a known method. Interested readers can refer to \cite{Gupta_Chen_2011} for a comprehensive tutorial. For image denoising, we note that the EM algorithm has several shortcomings as follows:

\begin{enumerate}
\item \textbf{Adaptivity}. For a fixed image database, the GMM parameters are specifically trained for that particular database. We call it the \emph{generic} parameter. If, for example, we are given an image that does not necessarily belong to the database, then it becomes unclear how one can adapt the generic parameter to the image.
\item \textbf{Computational cost}. Learning a good GMM requires a large number of training samples. For example, the GMM in \cite{Zoran_Weiss_2011} is learned from 2,000,000 randomly sampled patches. If our goal is to adapt a generic parameter to a particular image, then it would be more desirable to bypass the computationally intensive procedure.
\item \textbf{Finite samples}. When training samples are few, the learned GMM will be over-fitted; some components will even become singular. This problem needs to be resolved because a noisy image contains much fewer patches than a database of patches.
\item \textbf{Noise}. In image denoising, the observed image always contains noise. It is not clear how to mitigate the noise while running the EM algorithm.
\end{enumerate}

\section{EM Adaptation}
The proposed EM adaptation takes a generic prior and adapts it to create a specific prior using very few samples. Before giving the details of the EM adaptation, we first provide a toy example to illustrate the idea.

\subsection{Toy Example}
Suppose we are given two two-dimensional GMMs with two clusters in each GMM. From each GMM, we synthetically generate 400 data points with each point representing a 2D coordinate shown in Figure \ref{fig:GMM Adaptation} (a) and (b). Imagine that the data points in (a) come from an external database whereas the data points in (b) come from a clean image of interest.

With the two sets of data, we apply EM to learn GMM 1 and GMM 2. Since we have enough samples, both GMMs are estimated reasonably well as shown in (a) and (b). However, if we reduce the number of points in (b) to 20, then learning GMM 2 becomes problematic as shown in (c). Therefore, the question is this: Suppose we are given GMM 1 and only 20 data points from GMM 2, is there a way that we can transfer GMM 1 to the 20 data points so that we can approximately estimate GMM 2? This is the goal of EM adaptation. A result for this example is shown in (d).

\begin{figure}[t]
\def\iw{0.475\linewidth}
\begin{tabular}{cc}
\includegraphics[width=\iw]{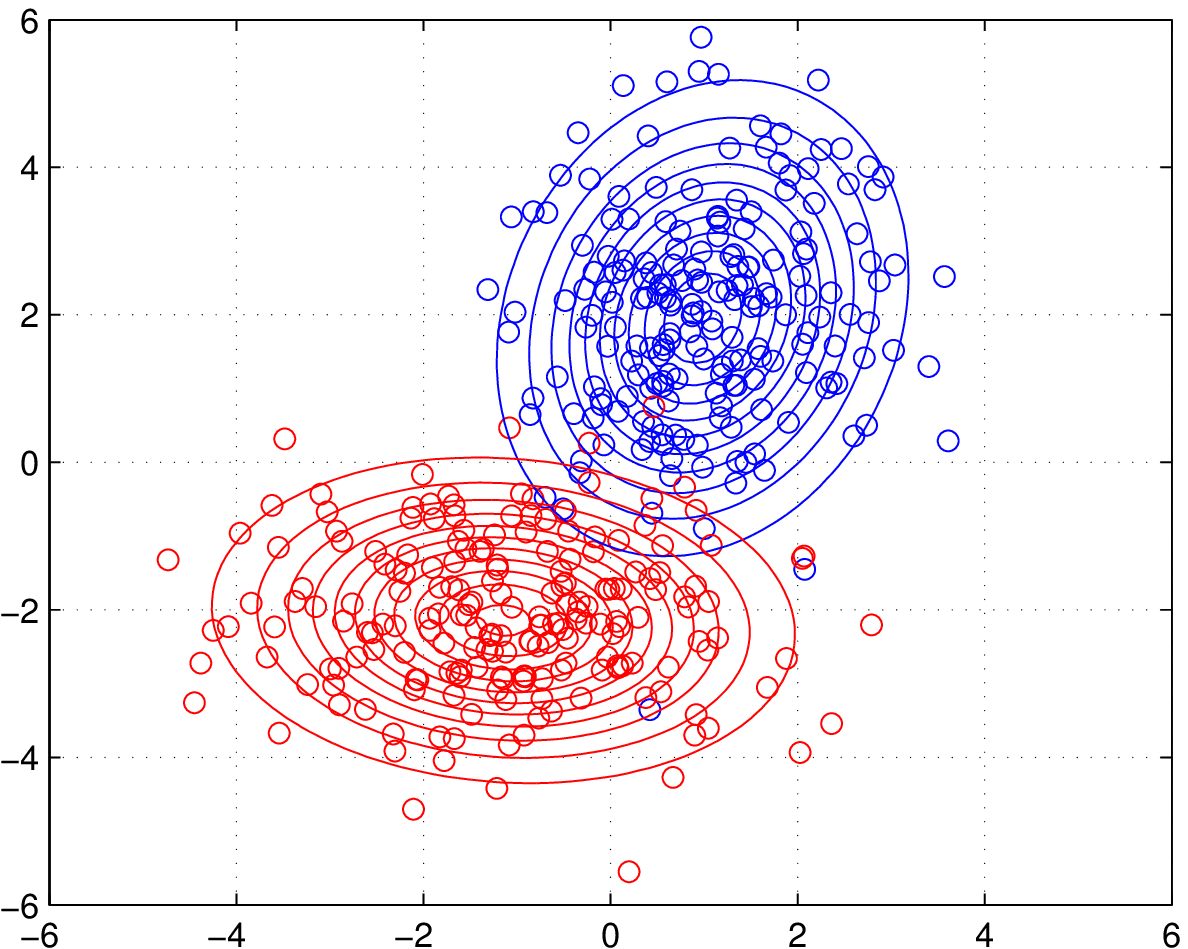}& \hspace{-3ex}
\includegraphics[width=\iw]{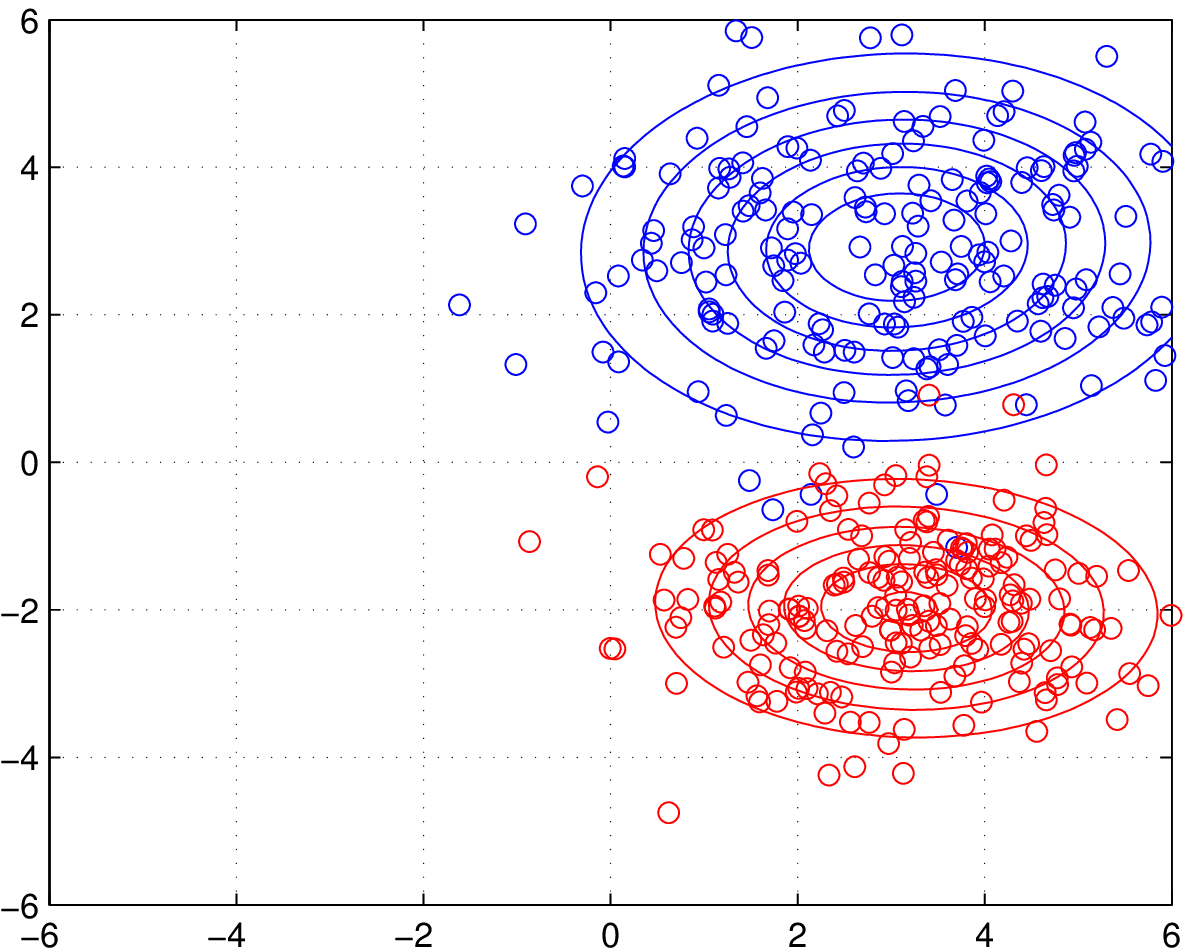}\\
(a) GMM 1 / 400 points & (b) GMM 2 / 400 points \\
\includegraphics[width=\iw]{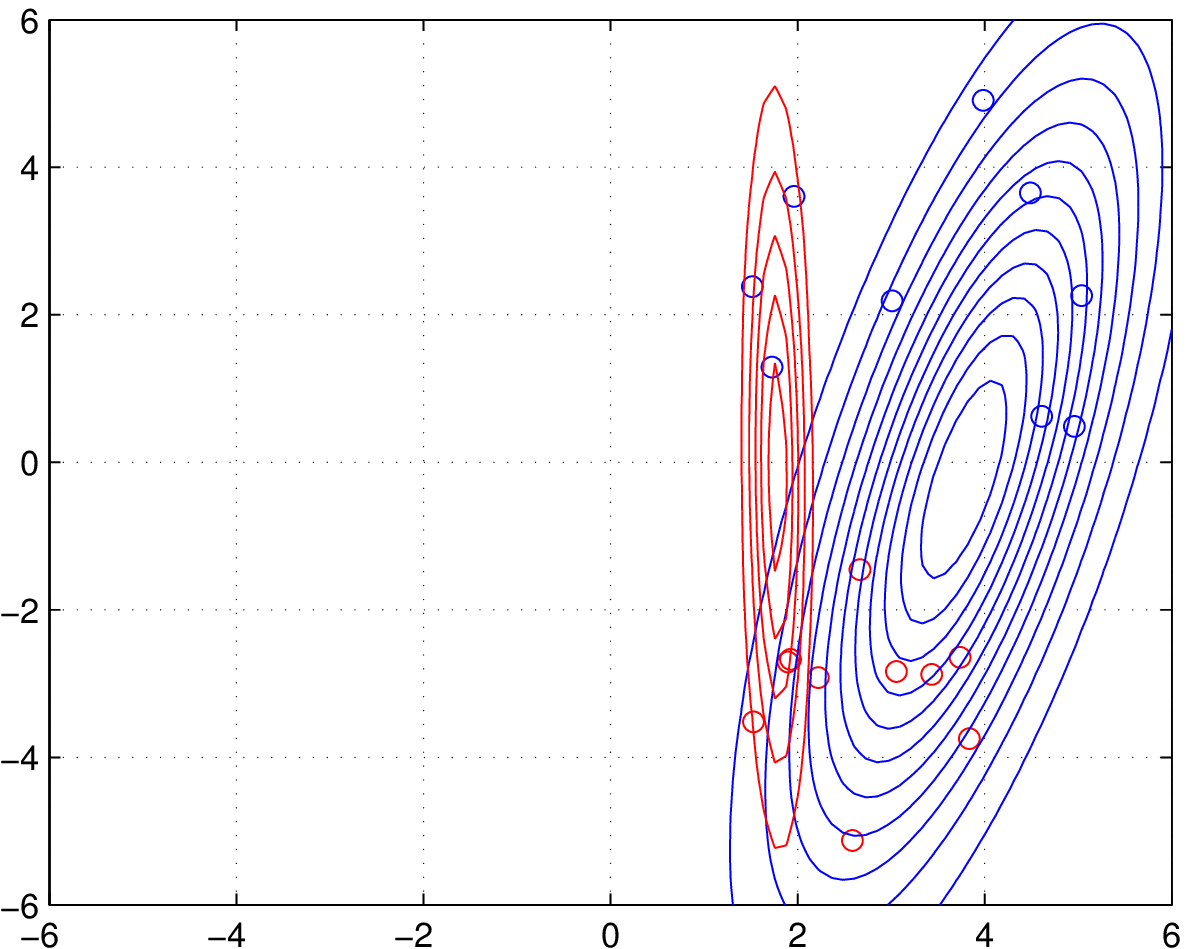}& \hspace{-3ex}
\includegraphics[width=\iw]{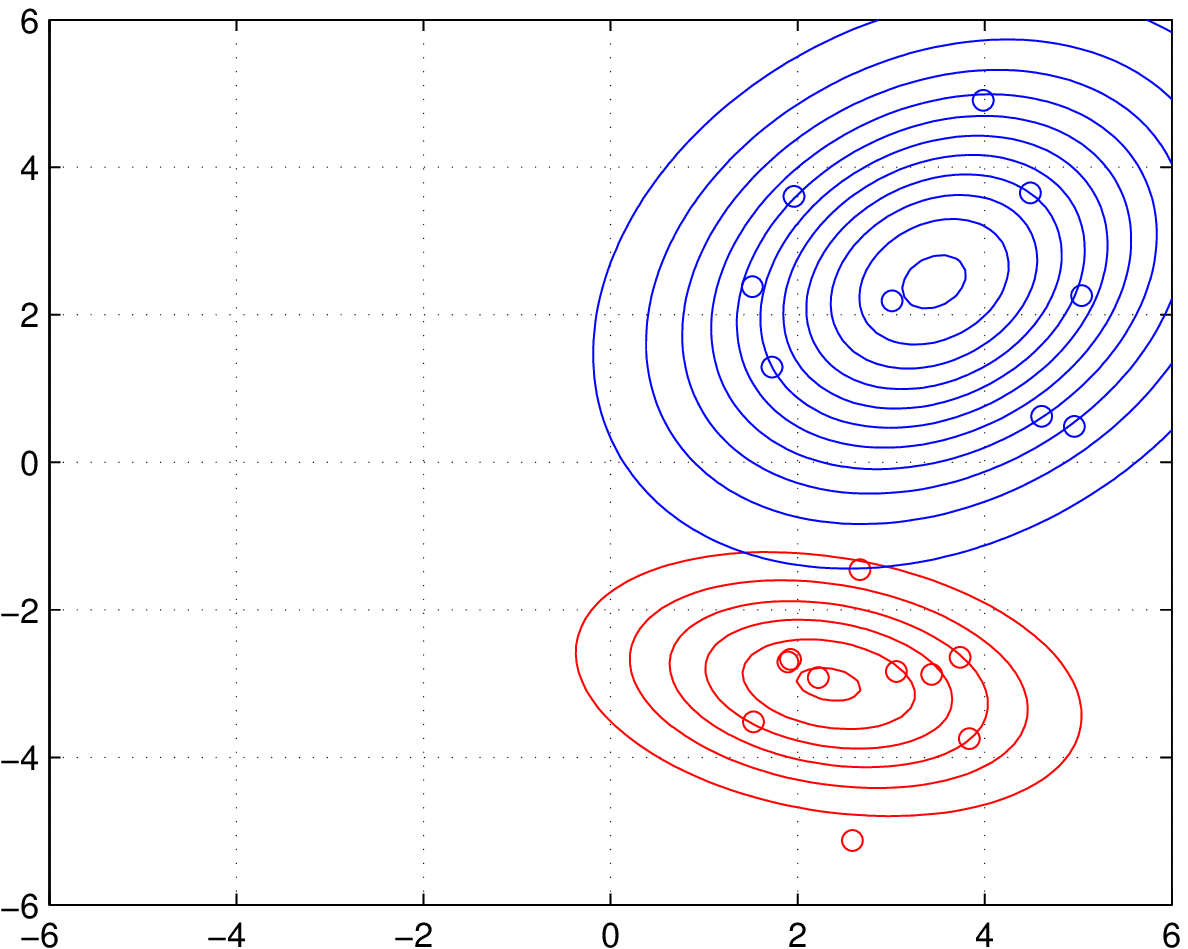}\\
(c) GMM 2 / 20 points & (d) Adapted / 20 points\\
\end{tabular}
\caption{(a) and (b), Two GMMs, each learned using the EM algorithm from 400 data points of 2D coordinates. (c): A GMM learned from a subset of 20 data points drawn from (b). (d): An adapted GMM using the same 20 data points in (c). Note the significant improvement from (c) to (d) by using the proposed adaptation.}
\label{fig:GMM Adaptation}
\end{figure}

\subsection{Bayesian Hyper-prior}
As illustrated in the toy example, what EM adaptation does is to use the generic model parameters as a ``guide'' when learning the new model parameters. Mathematically, suppose $\{\widetilde{\vp}_1,\ldots,\widetilde{\vp}_n\}$ are patches from a single image parameterized by a GMM with a parameter $\widetilde{\mTheta} \bydef \{(\widetilde{\pi}_k, \widetilde{\vmu}_k, \widetilde{\mSigma}_k)\}_{k=1}^K$. Our goal is to estimate $\widetilde{\mTheta}$ with the aid of some generic GMM parameter $\mTheta$. However, in order to formally derive the algorithm, we need to explain a Bayesian learning framework.

From a Bayesian perspective, estimation of the parameter $\widetilde{\mTheta}$ can be formulated as
\begin{align}
\widetilde{\mTheta}  &= \argmax{\widetilde{\mTheta}} \; \log{f( \widetilde{\mTheta} \,|\, \widetilde{\vp}_1,\ldots,\widetilde{\vp}_n )} \nonumber\\
&= \argmax{\widetilde{\mTheta}} \Big\{ \log{ f(\widetilde{\vp}_1,\ldots,\widetilde{\vp}_n \, |\, \widetilde{\mTheta})} + \log{f(\widetilde{\mTheta})} \Big\},
\label{hyper-prior: MAP estimation problem}
\end{align}
where
$$f
(\widetilde{\vp}_1,\ldots,\widetilde{\vp}_n \,|\, \widetilde{\mTheta})
= \prod_{i = 1}^n \left\{ \sum_{k = 1}^K \;\widetilde{\pi}_k \calN(\widetilde{\vp}_i | \widetilde{\vmu}_k, \widetilde{\mSigma}_k)\right\}
$$
is the joint distribution of the samples, and $f(\widetilde{\mTheta})$ is some prior of $\widetilde{\mTheta}$. We note that \eref{hyper-prior: MAP estimation problem} is also a MAP problem. However, the MAP for \eref{hyper-prior: MAP estimation problem} is the estimation of the model parameter $\widetilde{\mTheta}$, which is different from the MAP for denoising used in \eref{MAP formulation for patches}. Although the difference seems subtle, there is a drastically different implication that we should be aware of.

In \eref{hyper-prior: MAP estimation problem}, $f(\widetilde{\vp}_1, \ldots, \widetilde{\vp}_n \,|\, \widetilde{\mTheta})$ denotes the distribution of a collection of patches conditioned on the parameter $\widetilde{\mTheta}$. It is the likelihood of observing $\{\widetilde{\vp}_1, \ldots, \widetilde{\vp}_n\}$ given the model parameter $\widetilde{\mTheta}$. $f(\widetilde{\mTheta})$ is a distribution of the parameter, which is called hyper-prior in machine learning \cite{Bishop_2006}. Since $\widetilde{\mTheta}$ is the model parameter, the hyper-prior $f(\widetilde{\mTheta})$ defines the probability density of $\widetilde{\mTheta}$.

Same as the usual Bayesian modeling, hyper-priors are chosen according to a subjective belief. However, for efficient computation, hyper-priors are usually chosen as the \emph{conjugate priors} of the likelihood function $f(\widetilde{\vp}_1, \ldots, \widetilde{\vp}_n \,|\, \widetilde{\mTheta})$ so that the posterior distribution $f( \widetilde{\mTheta} \,|\, \widetilde{\vp}_1,\ldots,\widetilde{\vp}_n )$ has the same functional form as the prior distribution. For example; Beta distribution is a conjugate prior for a Bernoulli likelihood function; Gaussian distribution is a conjugate prior for a likelihood function that is also Gaussian, etc. For more discussions on conjugate priors, we refer the readers to \cite{Bishop_2006}.

\subsection{$f(\widetilde{\mTheta})$ for GMM}
For GMM, no joint conjugate prior can be found through the sufficient statistic approach \cite{Gauvain_Lee_1994}. To allow tractable computation, it is necessary to separately model the mixture weights and the means/covariances by assuming that the weights and means/covariances are independent.

We model the mixture weights as a multinomial distribution so that the corresponding conjugate prior for the mixture weight vector $(\widetilde{\pi}_1, \cdots, \widetilde{\pi}_K)$ is a Dirichlet density
\begin{equation}
\widetilde{\pi}_1, \cdots, \widetilde{\pi}_K \; \sim \; \mbox{Dir}(v_1, \cdots, v_k),
\label{Dirichlet density}
\end{equation}
where $v_i > 0$ is a pseudo-count for the Dirichlet distribution.

For mean and covariance $( \widetilde{\vmu}_k, \widetilde{\mSigma}_k )$, the conjugate prior is the normal-inverse-Wishart density so that
\begin{equation}
( \widetilde{\vmu}_k, \widetilde{\mSigma}_k) \sim
\mbox{NIW}(\vvartheta_k, \tau_k, \mPsi_k, \varphi_k), \;\; \mbox{for} \; k = 1, \cdots, K,
\label{normal-inverse-Wishart density}
\end{equation}
where $(\vvartheta_k, \tau_k, \mPsi_k, \varphi_k)$ are the parameters for the normal-inverse-Wishart density such that $\vvartheta_k$ is a vector of dimension $d$, $\tau_k > 0$, $\mPsi_k$ is a $d \times d$ positive definite matrix, and $\varphi_k > d - 1$.

\vspace{1ex}
\begin{remark}
The choice of the normal-inverse-Wishart distribution is important here, for it is the conjugate prior of a multivariate normal distribution with unknown mean and unknown covariance matrix. This choice is slightly different from \cite{Gauvain_Lee_1994} where the authors choose a normal-Wishart distribution. While both normal-Wishart and normal-inverse-Wishart can lead to the same result, the proof using normal-inverse-Wishart is considerably simpler for its inverted matrices.
\end{remark}

\vspace{1ex}

Assuming $\widetilde{\pi}_k$ is independent of $(\widetilde{\vmu}_k,\widetilde{\mSigma}_k)$, we factorize $f(\widetilde{\mTheta})$ as a product of \eref{Dirichlet density} and \eref{normal-inverse-Wishart density}. By ignoring the scaling constants, it is not difficult to show that
\begin{equation}
\begin{array}{l}
f( \widetilde{\mTheta} ) \propto \prod_{k = 1}^K \Big\{\widetilde{\pi}_k^{v_k - 1}|\widetilde{\mSigma}_k|^{-(\varphi_k + d + 2)/ 2} \\
\mbox{exp}\left( -\frac{\tau_k}{2} \left( \widetilde{\vmu}_k - \vvartheta_k \right)^T  \widetilde{\mSigma}_k^{-1}\left( \widetilde{\vmu}_k - \vvartheta_k \right) - \frac{1}{2} \mbox{tr}( \mPsi_k \widetilde{\mSigma}_k^{-1} ) \right) \Big\}.
\end{array}
\label{conjugate prior for GMM}
\end{equation}

The importance of \eref{conjugate prior for GMM} is that it is a conjugate prior of the complete data. As a result, the posterior density $f(\widetilde{\mTheta} | \widetilde{\vp}_1,\ldots,\widetilde{\vp}_n)$ belongs to the same distribution family as $f(\widetilde{\mTheta})$. This can be formally described in Proposition \ref{proposition:hyper-prior for the reproducing density}.

\vspace{1ex}
\begin{proposition}
\label{proposition:hyper-prior for the reproducing density}
Given the prior in \eref{conjugate prior for GMM}, the posterior $f(\widetilde{\mTheta} | \widetilde{\vp}_1,\ldots,\widetilde{\vp}_n)$ is given by
\begin{equation}
\begin{array}{l}
f(\widetilde{\mTheta} \,|\, \widetilde{\vp}_1,\ldots,\widetilde{\vp}_n) \propto \prod_{k = 1}^K \Big\{\widetilde{\pi}_k^{v_k' - 1}|\widetilde{\mSigma}_k|^{-(\varphi_k' + d + 2)/ 2} \\
\mbox{exp}\left( -\frac{\tau_k'}{2} \left( \widetilde{\vmu}_k - \vvartheta_k' \right)^T  \widetilde{\mSigma}_k^{-1}\left( \widetilde{\vmu}_k - \vvartheta_k' \right) - \frac{1}{2} \mbox{tr}( \mPsi_k' \widetilde{\mSigma}_k^{-1} ) \right) \Big\}
\end{array}
\label{hyper-prior for the reproducing density}
\end{equation}
where
\begin{align*}
v_k'            &= v_k + n_k, \;\; \varphi_k' = \varphi_k + n_k, \;\; \tau_k' = \tau_k + n_k,\\
\vvartheta_k'   &= \frac{\tau_k\vvartheta_k + n_k\bar{\vmu}_k}{\tau_k + n_k},\\
\mPsi_k'        &= \mPsi_k + \mS_k + \frac{\tau_k n_k}{\tau_k + n_k} (\vvartheta_k - \bar{\vmu}_k)(\vvartheta_k - \bar{\vmu}_k)^T,\\
\bar{\vmu}_k    &= \frac{1}{n_k} \sum_{i = 1}^n \gamma_{ki} \widetilde{\vp}_i, \;\; \mS_k = \sum_{i = 1}^n\gamma_{ki}(\widetilde{\vp}_i - \bar{\vmu}_k)(\widetilde{\vp}_i - \bar{\vmu}_k)^T
\end{align*}
are the parameters for the posterior density.
\end{proposition}

\begin{proof}
See Appendix \ref{appendix:proof,hyper-prior for the reproducing density}.
\end{proof}
\vspace{1ex}

\subsection{Solve for $\widetilde{\mTheta}$}
Solving for the optimal $\widetilde{\mTheta}$ is equivalent to solving the following optimization problem:
\begin{equation}
\begin{array}{ll}
\maximize{\widetilde{\mTheta}}  &\quad L(\widetilde{\mTheta}) \bydef\log{f(\widetilde{\mTheta} | \widetilde{\vp}_1,\ldots,\widetilde{\vp}_n)} \\
\subjectto      & \quad \sum_{k = 1}^K \widetilde{\pi}_k = 1.
\end{array}
\label{optimization for mTheta}
\end{equation}
The constrained problem \eref{optimization for mTheta} can be solved by considering the Lagrange function and taking derivatives with respect to each individual parameter. We summarize the optimal solutions in Proposition \ref{proposition: solution to optimization for mTheta}.
\vspace{1ex}
\begin{proposition}
\label{proposition: solution to optimization for mTheta}
The optimal $(\widetilde{\pi}_k, \widetilde{\vmu}_k, \widetilde{\mSigma}_k)$ for \eref{optimization for mTheta} are
\begin{align}
\widetilde{\pi}_k =&\; \frac{n}{(\sum_{k = 1}^K v_k - K) + n} \cdot \frac{n_k}{n} \nonumber\\
&\; + \frac{\sum_{k = 1}^K v_k - K}{(\sum_{k = 1}^K v_k - K) + n} \cdot \frac{v_k - 1}{\sum_{k = 1}^K v_k - K}, \label{unsimplified weight adaptation}\\
\widetilde{\vmu}_k =&\; \frac{1}{\tau_k + n_k}\sum_{i = 1}^n \gamma_{ki} \widetilde{\vp}_i + \frac{\tau_k}{\tau_k + n_k}\vvartheta_k, \label{unsimplified mean adaptation}\\
\widetilde{\mSigma}_k =&\; \frac{n_k}{\varphi_k + d + 2 + n_k} \frac{1}{n_k} \sum_{i = 1}^n \gamma_{ki}(\widetilde{\vp}_i - \widetilde{\vmu}_k)(\widetilde{\vp}_i - \widetilde{\vmu}_k)^T \nonumber \\
&+ \frac{1}{\varphi_k + d + 2 + n_k}\left( \mPsi_k + \tau_k(\vvartheta_k - \widetilde{\vmu}_k)(\vvartheta_k - \widetilde{\vmu}_k)^T \right). \label{unsimplified covariance adaptation}
\end{align}
\end{proposition}

\begin{proof}
See Appendix \ref{appendix:proof,solution to optimization for mTheta}.
\end{proof}
\vspace{1ex}

\begin{remark}
The results we showed in Proposition \ref{proposition: solution to optimization for mTheta} are different from \cite{Gauvain_Lee_1994}. In particular, the denominator for $\widetilde{\mSigma}_k$ in \cite{Gauvain_Lee_1994} is $\varphi_k - d + n_k$ whereas ours is $\varphi_k + d + 2 + n_k$. However, by using the simplification described in the next subsection, we can obtain the same result for both cases.
\end{remark}

\subsection{Simplification of $\widetilde{\mTheta}$}
The results in Proposition \ref{proposition: solution to optimization for mTheta} are general expressions for any hyper-parameters. We now discuss how to simplify the result with the help of the generic prior. First, since $\frac{v_k - 1}{\sum_{k = 1}^K v_k - K}$ is the mode of the Dirichlet distribution, a good surrogate for it is $\pi_k$. Second, $\vvartheta_k$ denotes the prior mean in the normal-inverse-Wishart distribution and thus can be appropriately approximated by $\vmu_k$. Moreover, since $\mPsi_k$ is the scale matrix on $\widetilde{\mSigma}_k$ and $\tau_k$ denotes the number of prior measurements in the normal-inverse-Wishart distribution, they can be reasonably chosen as $\mPsi_k = (\varphi_k + d + 2)\mSigma_k$ and $\tau_k = \varphi_k + d + 2$. Plugging these approximations in the results of Proposition \ref{proposition: solution to optimization for mTheta}, we summarize the simplification results as follows:

\vspace{1ex}
\begin{proposition}
\label{proposition:simplified adaptation}
Define $\rho \bydef \frac{n_k}{n}(\sum_{k = 1}^K v_k - K) = \tau_k = \varphi_k + d + 2$. Let
\begin{align*}
&\vvartheta_k = \vmu_k, \;\; \mPsi_k = (\varphi_k + d + 2)\mSigma_k, \;\; \frac{v_k - 1}{\sum_{k = 1}^K v_k - K}  = \pi_k,
\end{align*}
and $\alpha_k = \frac{n_k}{\rho + n_k}$, then \eref{unsimplified weight adaptation}-\eref{unsimplified covariance adaptation} become
\begin{align}
\widetilde{\pi}_k =& \alpha_k \frac{n_k}{n} + (1 - \alpha_k)\pi_k, \label{simplified weight adaptation}\\
\widetilde{\vmu}_k =& \alpha_k \frac{1}{n_k} \sum_{i = 1}^n \gamma_{ki} \widetilde{\vp}_i + (1 - \alpha_k)\vmu_k, \label{simplified mean adaptation}\\
\widetilde{\mSigma}_k =& \; \alpha_k \frac{1}{n_k} \sum_{i = 1}^n \gamma_{ki}(\widetilde{\vp}_i - \widetilde{\vmu}_k)(\widetilde{\vp}_i - \widetilde{\vmu}_k)^T \nonumber\\
 & + (1 - \alpha_k)\left( \mSigma_k + (\vmu_k - \widetilde{\vmu}_k)(\vmu_k - \widetilde{\vmu}_k)^T \right). \label{simplified covariance adaptation}
\end{align}
\end{proposition}

\begin{remark}
We note that Reynold \emph{et al.} \cite{Reynolds_Quatieri_Dunn_2000} presented similar simplification results (without derivations) as ours. However, their results are valid only for the scalar case or when the covariance matrices are diagonal. In contrast, our results support full covariance matrices and thus are more general. As will be seen, for our denoising application, since the image pixels (especially adjacent pixels) are correlated, the full covariance matrices are necessary for good performance.
\end{remark}

\begin{remark}
Comparing \eref{simplified covariance adaptation} with the work of Lu \emph{et al.} \cite{Lu_Lin_Jin_Yang_Wang_2015}, we note that in \cite{Lu_Lin_Jin_Yang_Wang_2015} the covariance is
\begin{equation}
\widetilde{\mSigma}_k = \alpha_k \frac{1}{n_k} \sum_{i = 1}^n \gamma_{ki}\widetilde{\vp}_i\widetilde{\vp}_i^T + (1 - \alpha_k)\mSigma_k.
\end{equation}
This result, although it looks similar to ours, is generally not valid if we follow the Bayesian hyper-prior approach, unless $\vmu_k$ and $\widetilde{\vmu}_k$ are both 0.
\end{remark}

\subsection{EM Adaptation Algorithm}
\begin{algorithm}[t]
\caption{EM adaptation Algorithm}
\label{alg:GMM adaptation algorithm}
\begin{algorithmic}
\STATE Input: $\mTheta = \{(\pi_k, \vmu_k, \mSigma_k)\}_{k=1}^K$, $\{\widetilde{\vp}_1,\ldots,\widetilde{\vp}_n\}$.
\STATE Output: Adapted parameters $\widetilde{\mTheta} = \{(\widetilde{\pi}_k, \widetilde{\vmu}_k, \widetilde{\mSigma}_k)\}_{k=1}^K$.
\STATE $\textbf{E-step}:$ Compute, for $k=1,\ldots,K$ and $i=1,\ldots,n$
\begin{align}
\gamma_{ki} &= \frac{\pi_k \calN(\widetilde{\vp}_i|\vmu_k, \mSigma_k)}{\sum\limits_{l=1}^K \pi_l \calN(\widetilde{\vp}_i|\vmu_l, \mSigma_l)}, \quad n_k =\sum\limits_{i=1}^n \gamma_{ki}. \label{posterior probabilities}
\end{align}
$\textbf{M-step}:$ Compute, for $k=1,\ldots,K$
\begin{align}
\widetilde{\pi}_k &= \alpha_k \frac{n_k}{n} + (1-\alpha_k)\pi_k, \label{weight adaptation}\\
\widetilde{\vmu}_k &= \alpha_k\frac{1}{n_k}\sum_{i=1}^n\gamma_{ki}\widetilde{\vp}_i + (1 - \alpha_k) \vmu_k, \label{mean adaptation}\\
\widetilde{\mSigma}_k &= \;\alpha_k\frac{1}{n_k} \sum_{i=1}^n\gamma_{ki}(\widetilde{\vp}_i - \widetilde{\vmu}_k)(\widetilde{\vp}_i - \widetilde{\vmu}_k)^T \nonumber\\
&\quad + (1 - \alpha_k) \left( \mSigma_k + (\vmu_k - \widetilde{\vmu}_k)(\vmu_k - \widetilde{\vmu}_k)^T \right).
\label{covariance adaptation}
\end{align}
\STATE Postprocessing: Normalize $\{\widetilde{\pi}_k\}_{k=1}^K$ so that they sum to 1,
\STATE \quad\quad and ensure $\{\widetilde{\mSigma}_k\}_{k=1}^K$ is positive semi-definite.
\end{algorithmic}
\end{algorithm}

The proposed EM adaptation algorithm is summarized in Algorithm~\ref{alg:GMM adaptation algorithm}. EM adaptation shares many similarities with the standard EM algorithm. To better understand the differences, we take a closer look at each step.

\vspace{2ex}
\textbf{E-Step}: The E-step in the EM adaptation is the same as in the EM algorithm. We compute the likelihood of $\widetilde{\vp}_i$ conditioned on the generic parameter $(\pi_k,\vmu_k,\mSigma_k)$ as
\begin{equation}
\gamma_{ki} = \frac{\pi_k \calN(\widetilde{\vp}_i \,|\,\vmu_k, \mSigma_k)}{\sum_{l=1}^K \pi_l \calN(\widetilde{\vp}_i \,|\,\vmu_l, \mSigma_l)}.
\label{eq:gamma}
\end{equation}

\vspace{2ex}
\textbf{M-Step}:
The M-step is a more interesting step. From \eref{weight adaptation} to \eref{covariance adaptation}, $(\widetilde{\pi}_k, \widetilde{\vmu}_k, \widetilde{\mSigma}_k)$ are updated through a linear combination of the contributions from the new data and the generic parameters. On one extreme, when $\alpha_k = 1$, the M-step turns exactly back to the M-step in the EM algorithm. On the other extreme, when $\alpha_k = 0$, all emphasis is put on the generic parameters. For $\alpha_k$ that lies in between, the updates are a weighted averaging of the new data and the generic parameters. Taking the mean as an example, the EM adaptation updates the mean according to
\begin{equation}
\widetilde{\vmu}_k = \underbrace{\alpha_k \left(\frac{1}{n_k}\sum_{i=1}^n\gamma_{ki}\widetilde{\vp}_i\right)}_\textbf{new data} + \underbrace{\vphantom{\sum_{i=1}^n} (1 - \alpha_k) \vmu_k}_\textbf{generic prior}.
\label{eq:udpate mean}
\end{equation}
The updated mean in \eref{eq:udpate mean} is a linear combination of two terms, where the first term is an empirical data average with the fractional weight $\gamma_{ki}$ from each data point $\widetilde{\vp}_i$ and the second term is the generic mean $\vmu_k$. Similarly for the covariance update in \eref{covariance adaptation}, the first term computes an empirical covariance with each data point weighted by $\gamma_{ki}$ which is the same as in the M-step of the EM algorithm, and the second term includes the generic covariance along with an adjustment term $(\vmu_k - \widetilde{\vmu}_k)(\vmu_k - \widetilde{\vmu}_k)^T$. These two terms are then linearly combined to yield the updated covariance.


\subsection{Convergence}
The EM adaptation shown in Algorithm \ref{alg:GMM adaptation algorithm} is an EM algorithm. Therefore, its convergence is guaranteed by the classical theory, which we state without proof as follows:

\begin{proposition}
\label{proposition: convergence proof}
Let $L(\widetilde{\mTheta}) = \log{f(\widetilde{\vp}_1,\ldots,\widetilde{\vp}_n | \widetilde{\mTheta})}$ be the log-likelihood, $f(\widetilde{\mTheta})$ be the prior, and $Q\big(\widetilde{\mTheta} | \widetilde{\mTheta}^{(m)} \big)$ be the Q function in the $m$-th iteration. If
\begin{equation*}
Q\big(\widetilde{\mTheta} | \widetilde{\mTheta}^{(m)}\big) + \log{f(\widetilde{\mTheta})} \geq Q\big(\widetilde{\mTheta}^{(m)} | \widetilde{\mTheta}^{(m)}\big) + \log{f\big(\widetilde{\mTheta}^{(m)}\big)},
\end{equation*}
then
\begin{equation*}
L(\widetilde{\mTheta}) + \log{f(\widetilde{\mTheta})} \geq L\big(\widetilde{\mTheta}^{(m)}\big) + \log{f\big(\widetilde{\mTheta}^{(m)}\big)}.
\end{equation*}
\end{proposition}

\begin{proof}
See \cite{Gupta_Chen_2011}.
\end{proof}

The convergence result states that the EM adaptation converges to a local minimum defined by the Q function. Experimentally, we observe that EM adaptation enters the steady state in a few iterations. To demonstrate this observation, we conduct experiments on different testing images. Figure \ref{fig:iterative adaptation effect} shows the result of one testing image. For all noise levels ($\sigma = 20$ to $100$), PSNR increases as more iterations are applied and converges after about four iterations. We also observe that for most testing images, the improvement becomes marginal after a single iteration.


\begin{figure}[h]
\def\iw{1\linewidth}
\includegraphics[width=\iw]{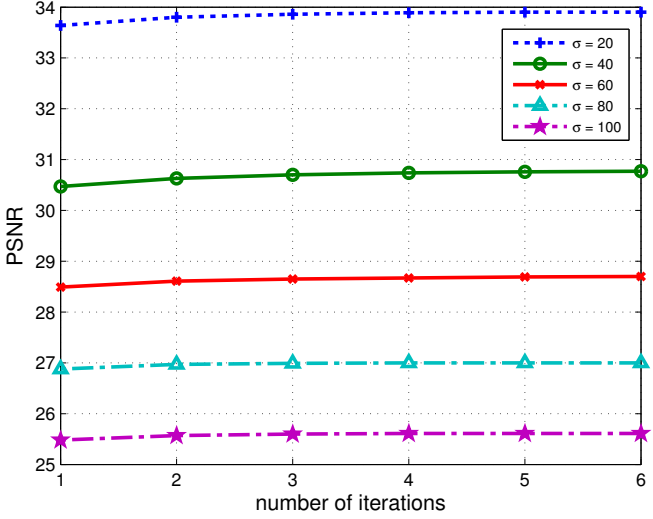}
\caption{Image denoising using EM adaptation: The PSNR only improves marginally after the first iteration, confirming that the EM adaptation can typically be performed in a single iteration. The testing image is House of size 256 $\times$ 256. $\sigma = 20, \ldots, 100$ indicates the noise level.}
\label{fig:iterative adaptation effect}
\end{figure}

\section{EM Adaptation for Denoising}
The proposed Algorithm \ref{alg:GMM adaptation algorithm} presented in the previous section works well when the training patches $\{\widetilde{\vp}_1, \ldots, \widetilde{\vp}_n\}$ are drawn from the \emph{clean} ground-truth image $\vx$. In this section, we discuss how to modify the EM adaptation algorithm when the data are drawn from noisy images. We will also discuss methods to improve computational complexity.

\subsection{Adaptation to Noisy Image}
In the presence of noise, the Gaussian mixture model is perturbed by the noise covariance. More precisely, if we assume that a clean patch $\widetilde{\vp}_i$ is drawn from a Gaussian mixture and if the noise $\vepsilon$ is i.i.d. Gaussian with distribution $\calN(0,\widetilde{\sigma}^2\mI)$, then the probability of drawing a noisy version $\widetilde{\vp}_i + \vepsilon$ is the convolution of the GMM and the Gaussian noise
\begin{equation*}
f(\widetilde{\vp}_i + \vepsilon) = \sum_{k=1}^K \pi_k \int_{-\infty}^{\infty}\calN(\widetilde{\vp}_i-\vq \,|\, \vmu_k, \mSigma_k) \calN(\vq \,|\, 0, \widetilde{\sigma}^2\mI) d\vq.
\end{equation*}
Since convolution is linear, it does not influence the weight $\pi_k$ and the mean $\vmu_k$. The covariance is perturbed by adding $\widetilde{\sigma}^2\mI$ to $\mSigma_k$. Therefore, the final distribution of $\widetilde{\vp}_i + \vepsilon$ is
\begin{equation*}
f(\widetilde{\vp}_i + \vepsilon) = \sum_{k=1}^K \pi_k \calN(\widetilde{\vp}_i \,|\, \vmu_k, \mSigma_k + \widetilde{\sigma}^2\mI).
\end{equation*}

The implication of the above analysis is that in the presence of noise, we should compensate the noise present in the observed data. As a result, the EM algorithm can be accordingly modified by changing \eref{posterior probabilities} to
\begin{equation}
\gamma_{ki} = \frac{\pi_k \calN(\widetilde{\vp}_i \,|\, \vmu_k, \mSigma_k + \widetilde{\sigma}^2\mI)}{\sum_{l=1}^K \pi_l \calN(\widetilde{\vp}_l \,|\, \vmu_l, \mSigma_l + \widetilde{\sigma}^2\mI)},
\label{revised posterior probabilities}
\end{equation}
and \eref{covariance adaptation} to
\begin{align}
\widetilde{\mSigma}_k = & \;\alpha_k\frac{1}{n_k} \sum_{i=1}^n\gamma_{ki}\left( (\widetilde{\vp}_i - \widetilde{\vmu}_k)(\widetilde{\vp}_i - \widetilde{\vmu}_k)^T - \widetilde{\sigma}^2\mI \right) \nonumber\\
&\quad + (1 - \alpha_k) \left( \mSigma_k + (\vmu_k - \widetilde{\vmu}_k)(\vmu_k - \widetilde{\vmu}_k)^T \right).
\label{revised covariance adaptation}
\end{align}
In other words, we add $\widetilde{\sigma}^2\mI$ to $\gamma_{ki}$ to ensure the correct likelihood of drawing a noisy observation, and subtract $\widetilde{\sigma}^2\mI$ from $\widetilde{\mSigma}_k$ to ensure the correct mixture model of the clean sample.

\begin{remark}
Our noise handling method shares some similarity with the work of Awate and Whitaker \cite{Awate_Whitaker_2007} by subtracting the noise variance from the observed variance. However, in \cite{Awate_Whitaker_2007} the variance is assumed identical for all mixture components and is a scalar. In our case, we have to estimate a collection of covariance matrices $\{\widetilde{\mSigma}_k\}_{k=1}^K$, which are dense. Another difference is that we use an adaptation approach by combining the generic prior and the new data. Awate and Whitaker \cite{Awate_Whitaker_2007} do not have generic priors, but they minimize the Kullback-Leibler divergence of the ideal distribution and the estimated distribution.
\end{remark}

\subsection{Adaptation by Pre-filtering}
A limitation of the above noise handling technique is that when $\widetilde{\sigma}$ is large, it becomes impossible for \eref{revised covariance adaptation} to return the true covariance as we only have limited number of samples for training. Therefore, for practical reasons we consider a pre-filtering step. This idea is similar to many image denoising algorithms such as BM3D \cite{Dabov_Foi_Katkovnik_2007} where a first stage pilot denoising is applied before going to the second stage of the actual denoising. In our problem, we apply an existing denoising algorithm to obtain a pre-filtered image. The adaptation is then applied to the pre-filtered image to generate an adapted prior. In the second stage, we apply the MAP denoising as described in Section II-A to obtain the final denoised image.

When a pre-filtering step is employed, we must address the question of what is the influence of a particular pre-filtering method. Clearly, a strong pre-filtering will likely lead to better final performance as the pre-filtered image provides more relevant information for the adaptation step. However, as we will see in the experiment section (Section~\ref{section5}), using a reasonably advanced denoising algorithm for pre-filtering does not cause much performance difference. The intuition is that pre-filtering is nothing but a black-box method to reduce the noise level $\widetilde{\sigma}$. Since $\widetilde{\sigma}$ can never be zero for any pre-filtering method, we will have to use \eref{revised posterior probabilities} and \eref{revised covariance adaptation} by replacing $\widetilde{\sigma}$ with the amount of noise remaining in the pre-filtered image. For most state-of-the-art denoising algorithms, the difference in $\widetilde{\sigma}$ is quite small. The real challenge is how to estimate $\widetilde{\sigma}$.

\subsection{Estimating $\widetilde{\sigma}^2$}
In order to analyze the noise remaining in the pre-filtered image, we let $\vxbar$ be the pre-filtered image. The distribution of the residue $\vxbar-\vx$ is unknown, but empirically we observe that it can be approximated by a single Gaussian as in \cite{Papyan_Elad_2016}. This approximation allows us to model $(\vxbar - \vx) \sim \calN(\vzero, \widetilde{\sigma}^2\mI)$, where $\widetilde{\sigma}^2 \bydef \E\|\vxbar - \vx\|^2$ is the variance of $\vxbar$. In other words, $\widetilde{\sigma}^2$ is the mean squared error (MSE) of $\vxbar$ compared to $\vx$. Therefore, if $\vx$ is available, estimating $\widetilde{\sigma}^2$ is trivial as it is just the MSE. However, in the absence of $\vx$, we need a surrogate to estimate the MSE and hence $\widetilde{\sigma}^2$.

The surrogate strategy we use is the Stein's Unbiased Risk Estimator (SURE) \cite{Stein_1981}. SURE provides a way for unbiased estimation of the true MSE. The analytical expression of SURE is
\begin{equation}
\widetilde{\sigma}^2 \approx \mbox{SURE} \bydef \frac{1}{n} \|\vy -\vxbar\|^2 - \sigma^2 + \frac{2\sigma^2}{n} \mbox{div},
\end{equation}
where div denotes the divergence of the denoising algorithm with respect to the noisy measurements. However, not all denoising algorithms have a closed form for the divergence term. To alleviate this issue, we adopt the Monte-Carlo SURE \cite{Ramani_Blu_Unser_2008} to approximate the divergence. We refer readers to \cite{Ramani_Blu_Unser_2008} for detailed discussions about Monte-Carlo SURE. The key steps are summarized in Algorithm \ref{alg:Monte-Carlo SURE algorithm}.

\begin{algorithm}[!]
\caption{Monte-Carlo SURE for Estimating $\widetilde{\sigma}^2$}
\label{alg:Monte-Carlo SURE algorithm}
\begin{algorithmic}
\STATE Input: $\vy \in \R^n$, $\sigma^2$, $\delta$ (typically $\delta = 0.01$).
\STATE Output: $\widetilde{\sigma}^2$.
\STATE Generate $\vb \sim \calN(\vzero, \mI) \in \R^n$.
\STATE Construct $\vy' = \vy + \delta \vb$.
\STATE Apply a denoising algorithm on $\vy$ and $\vy'$  to get two pre-filtered images $\vxbar$ and $\vxbar'$, respectively.
\STATE Compute $\mbox{div} = \frac{1}{\delta} \vb^T(\vxbar' - \vxbar)$.
\STATE Compute $ \widetilde{\sigma}^2 = \mbox{SURE} \bydef \frac{1}{n} \|\vy -\vxbar\|^2 - \sigma^2 + \frac{2\sigma^2}{n} \mbox{div}$.
\end{algorithmic}
\end{algorithm}

\begin{figure}[h]
\def\iw{1\linewidth}
\includegraphics[width=\iw]{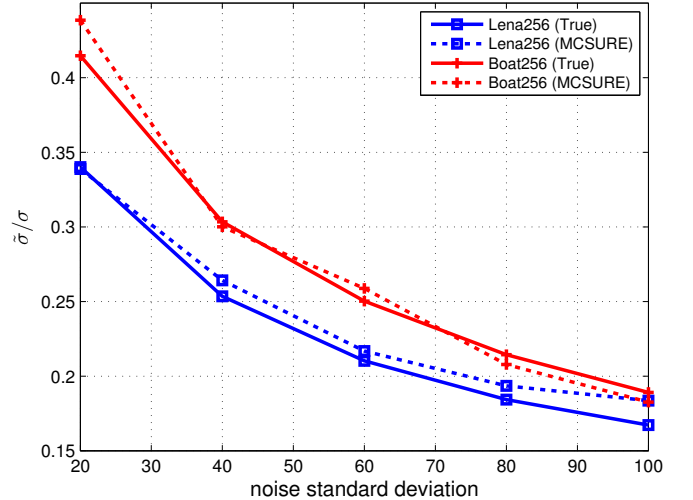}
\caption{Comparison between the true MSE and Monte-Carlo SURE when estimating $\widetilde{\sigma}/\sigma$ over a large range of noise levels. The pre-filtering method is EPLL. }
\label{fig:monte-carlo SURE vs true MSE}
\end{figure}

To demonstrate the effectiveness of Monte-Carlo SURE, we compare the estimates for $\widetilde{\sigma} / \sigma$ when we use the true MSE and Monte-Carlo SURE. As is observed in Figure \ref{fig:monte-carlo SURE vs true MSE}, over a large range of noise levels, the Monte-Carlo SURE curves are quite similar to the true MSE curves. The pre-filtering method in Figure \ref{fig:monte-carlo SURE vs true MSE} is EPLL. For other methods, such as BM3D, we have similar observations for different noise levels.

\subsection{Estimating $\alpha_k$}
Besides the pre-filtering for noisy images, we should also determine the combination weight $\alpha_k$ for the EM adaptation. From the derivation of the algorithm, the combination weight $\alpha_k = \frac{n_k}{n_k + \rho}$ is determined by both the probabilistic count $n_k$ and the relevance factor $\rho$. The factor $\rho$ is adjusted to allow different adaptation rates. When there is a good match between training data and target noisy image, we can let the adaptation rely more on the training data; When there is a poor match between training data and target noisy image, we use more of the noisy image. This parameter is user defined, similar to a regularization parameter in optimization. In the application of speaker verification \cite{Reynolds_Quatieri_Dunn_2000,Woodland_2001}, $\rho$ is set to 16. Experimentally, we find that the performance is insensitive to $\rho$ in the range of 8 and 20.

In this paper, $\rho$ is tuned to maximize the PSNR of the denoised image. In Figure \ref{fig:rho effect on an denoised image}, we show how PSNR changes in terms of $\rho$. The PSNR curves indicate that for a testing image of $64 \times 64$, a large $\rho$ for EM adaptation is better. As the testing images become larger, we observe that the optimal $\rho$ becomes smaller. Empirically, we find that $\rho$ in the range of 1 and 10 works well for a variety of images (over $200 \times 200$) for different noise levels.

\begin{remark}
For a fixed $\rho$, it is perhaps tempting to compare the estimated GMM with the ground truth GMM because a good match will likely provide better denoising performance. However, we note that the ground truth GMM is never available even if we use the oracle clean image because a finite-size image does not have adequate patches for estimating the GMM parameters.
\end{remark}

\begin{figure}[h]
\def\iw{1\linewidth}
\includegraphics[width=\iw]{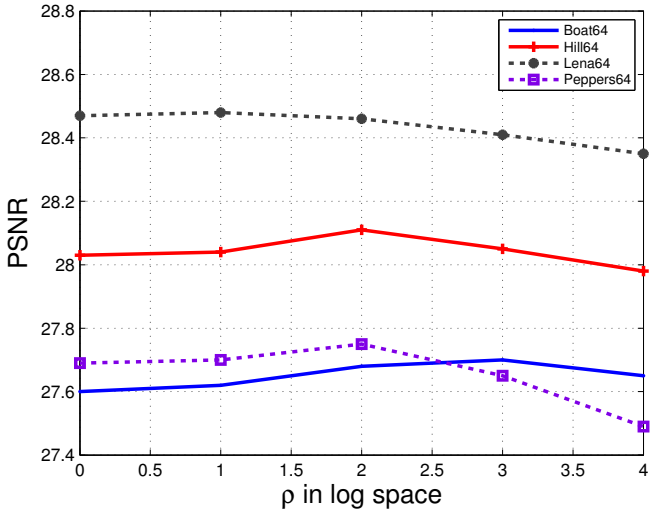}
\caption{The effect of $\rho$ on denoising performance. The pre-filtered image is used for the EM adaptation algorithm. The testing images are of size $64 \times 64$ with noise $\sigma = 20$.}
\label{fig:rho effect on an denoised image}
\end{figure}

\subsection{Computational Improvement}
Finally, we comment on a simple but very effective way of improving the computational speed. If we take a closer look at the M-step in Algorithm \ref{alg:GMM adaptation algorithm}, we observe that $\widetilde{\pi}_k$ and $\widetilde{\vmu}_k$ are easy to compute. However, $\widetilde{\mSigma}_k$ is time-consuming to compute, because updating each of the $K$ covariance matrices requires $n$ time-consuming outer product operations $\sum_{i=1}^n \gamma_{ki}(\widetilde{\vp}_i - \widetilde{\vmu}_k)(\widetilde{\vp}_i - \widetilde{\vmu}_k)^T$. Most previous works mitigate the problem by assuming that the covariance is diagonal \cite{Reynolds_Quatieri_Dunn_2000,
Woodland_2001,Dixit_Rasiwasia_Vasconcelos_2011}. However, this assumption is not valid in our case because image pixels (especially neighboring pixels) are correlated.

Our solution to this problem is shown in Proposition \ref{proposition:simplified covariance adaptation}. The new result is an \emph{exact} computation of \eref{covariance adaptation} but with significantly less operations. The idea is to exploit the algebraic structure of the covariance matrix.

\begin{proposition}
\label{proposition:simplified covariance adaptation}
The full covariance adaptation in \eref{covariance adaptation} can be simplified as
\begin{align}
\widetilde{\mSigma}_k = & \;\alpha_k \frac{1}{n_k}\sum_{i=1}^n \gamma_{ki}\widetilde{\vp}_i\widetilde{\vp}_i^T - \widetilde{\vmu}_k\widetilde{\vmu}_k^T \nonumber\\
& \; +(1-\alpha_k)(\mSigma_k + \vmu_k\vmu_k^T).
\label{eq:simplified covariance adaptation}
\end{align}
\end{proposition}

\begin{proof}
See Appendix \ref{appendix:proof,simplified covariance adaptation}.
\end{proof}

The simplification is very rewarding because computing $\alpha_k \frac{1}{n_k}\sum_{i=1}^n \gamma_{ki}\widetilde{\vp}_i\widetilde{\vp}_i^T$ does not involve $\widetilde{\vmu}_k$ and thus can be pre-computed for each component, which makes the computation of $\widetilde{\mSigma}_k$ much more efficient. This reduces the computational complexity from $O(nd^2)$ for \eref{covariance adaptation} down to $O(d^2)$ for \eref{eq:simplified covariance adaptation}. In Table \ref{table:runtime comparison for covariance simplification}, we list the averaging runtime when computing \eref{covariance adaptation} and \eref{eq:simplified covariance adaptation} for two image sizes.

\begin{table}[h]
\centering
\begin{tabular}{cc|c|c|c}
\hline
& \small image size & Eq. \eref{covariance adaptation} & Eq. \eref{eq:simplified covariance adaptation} & Speedup\\
&                   & (original)                       & (ours)                                             & factor\\
\hline
runtime (sec) & ($64 \times 64$) & 31.34 & 0.30 & 104.5\\
runtime (sec) & ($128 \times 128$) & 136.58 & 1.32 & 103.2\\
\hline
\end{tabular}
\caption{Runtime comparison between \eref{covariance adaptation} and \eref{eq:simplified covariance adaptation} for different image sizes.}
\label{table:runtime comparison for covariance simplification}
\end{table}

\section{Experimental Results}
\label{section5}
In this section, we present experimental results for \emph{single}- and \emph{example}-based image denoising. \emph{Single} refers to using the single noisy image for training, whereas \emph{example} refers to using an external reference image for training.

\subsection{Experiment Settings}
For comparison, we consider two state-of-the-art methods: BM3D \cite{Dabov_Foi_Katkovnik_2007} and EPLL \cite{Zoran_Weiss_2011}. For both methods, we run the original codes provided by the authors with the default parameters. The GMM prior in EPLL is learned from 2,000,000 randomly chosen $8 \times 8$ patches. The number of mixture components is fixed at 200 to match the GMM prior provided by EPLL. Alternatively, we can optimize the number of mixtures by performing a cross validation against the noise distribution as in \cite{Chan_Zickler_Lu_2016} or performing a distribution match against the generic prior as in \cite{Lu_Lin_Jin_Yang_Wang_2015}.

We consider three versions of EM adaptation: (1) an oracle adaptation by adapting the generic prior to the ground-truth image, denoted as \emph{aGMM-clean}; (2) a pre-filtered adaptation by adapting the generic prior to the EPLL result, denoted as \emph{aGMM-EPLL}; (3) a pre-filtered adaptation by adapting the generic prior to the BM3D result, denoted as \emph{aGMM-BM3D}. In the example-based image denoising, we adapt the generic prior to an example image and denote it as \emph{aGMM-example}. We set the parameter $\rho = 1$, and experimental results show that the performance is insensitive to $\rho$ being in the range of 1 and 10. We run denoising experiments on a variety of images and for a large range of noise standard deviations ($\sigma=20, 40, 60, 80, 100$). To reduce the bias due to a particular noise realization, each reported PSNR result is averaged over at least eight independent trials.

\subsection{Single Image Denoising}
We use six standard images of size $256 \times 256$, and six natural images of size $481 \times 321$ randomly chosen from \cite{Martin_Fowlkes_Tal_Malik_2001} for the single image denoising experiments.

\begin{figure*}[!]
\centering
\def\ih{0.18\linewidth}
\def\iw{0.21\linewidth}
\def\sw{0.07\linewidth}
\begin{tabular}{cccccc}
noisy image & BM3D & EPLL & aGMM & ground\\
 & & & -EPLL & -truth\\
\includegraphics[height=\ih,width=\ih]{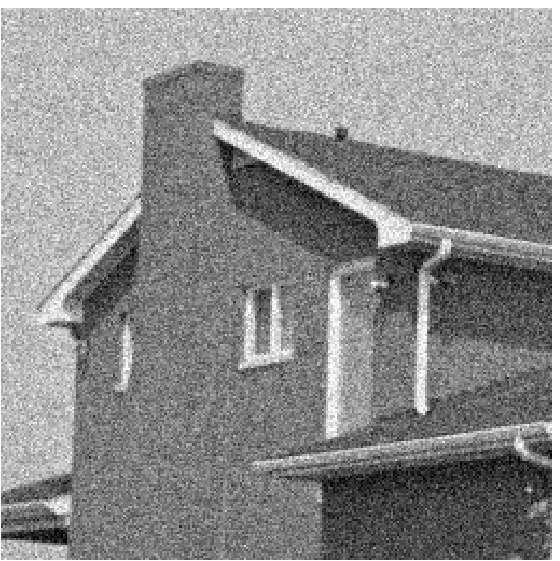}&
\includegraphics[height=\ih,width=\iw]{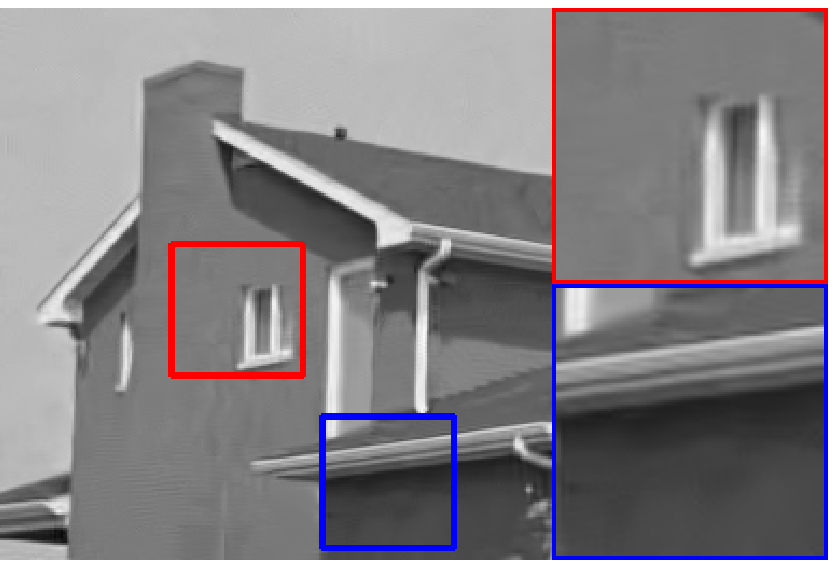}&
\includegraphics[height=\ih,width=\iw]{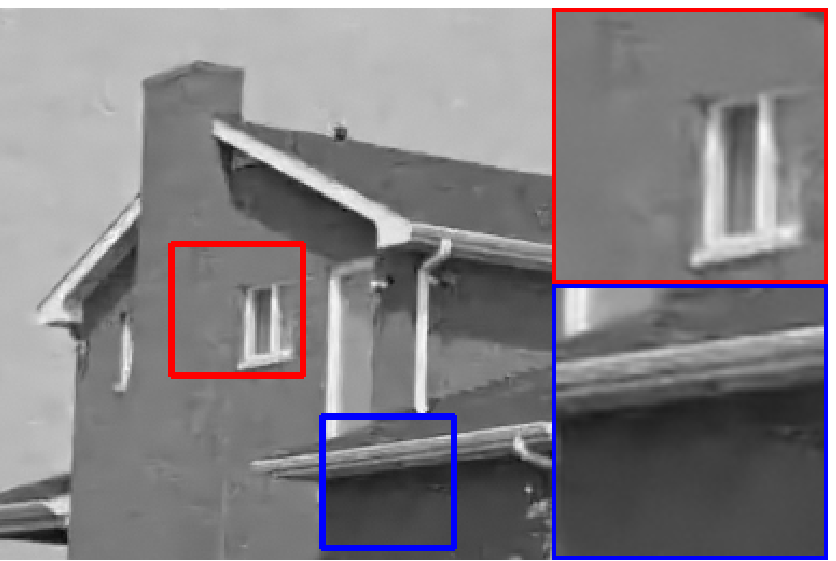}&
\includegraphics[height=\ih,width=\iw]{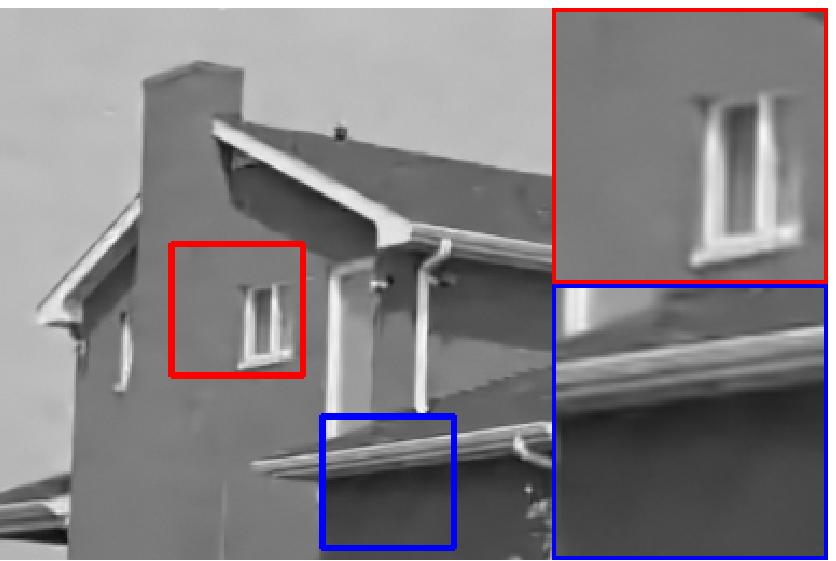}&
\includegraphics[height=\ih,width=\sw]{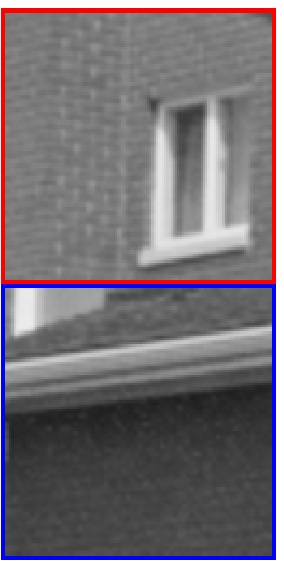}\\
$\sigma = 20$ & 33.67 dB (0.8709) & 33.03 dB (0.8618) & 33.63 dB (0.8671)\\

\includegraphics[height=\ih,width=\ih]{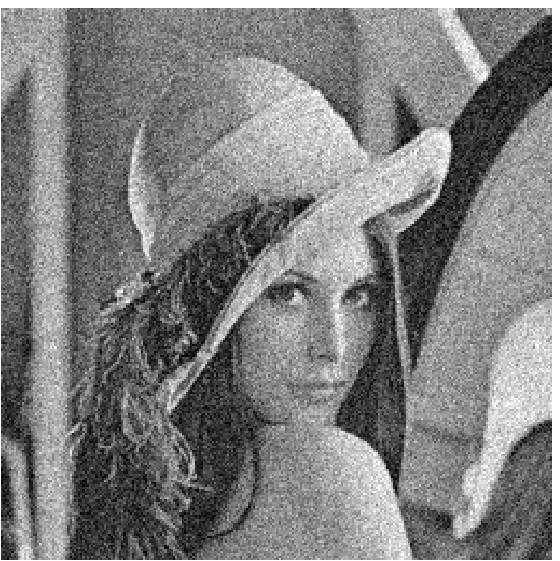}&
\includegraphics[height=\ih,width=\iw]{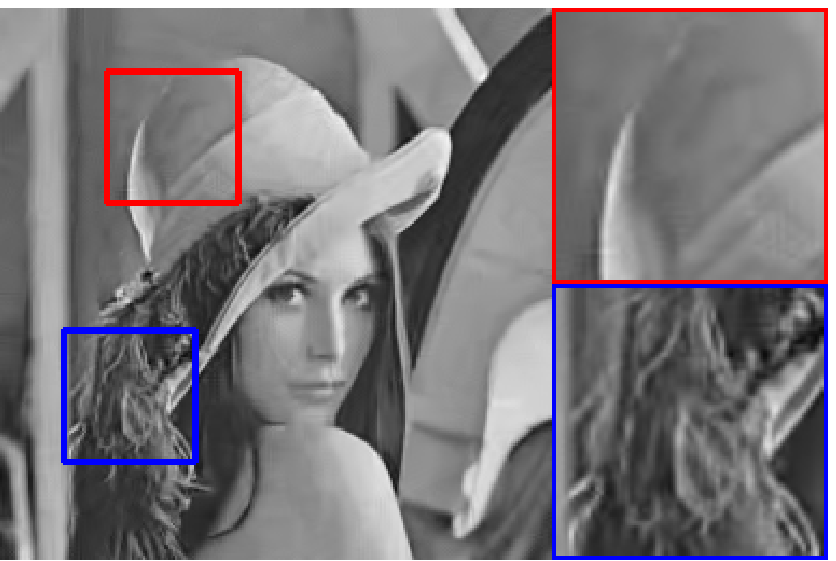}&
\includegraphics[height=\ih,width=\iw]{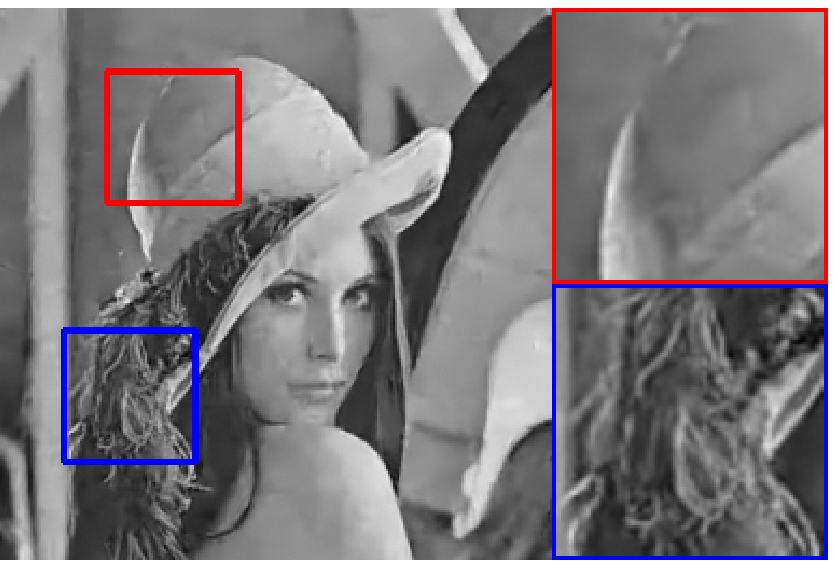}&
\includegraphics[height=\ih,width=\iw]{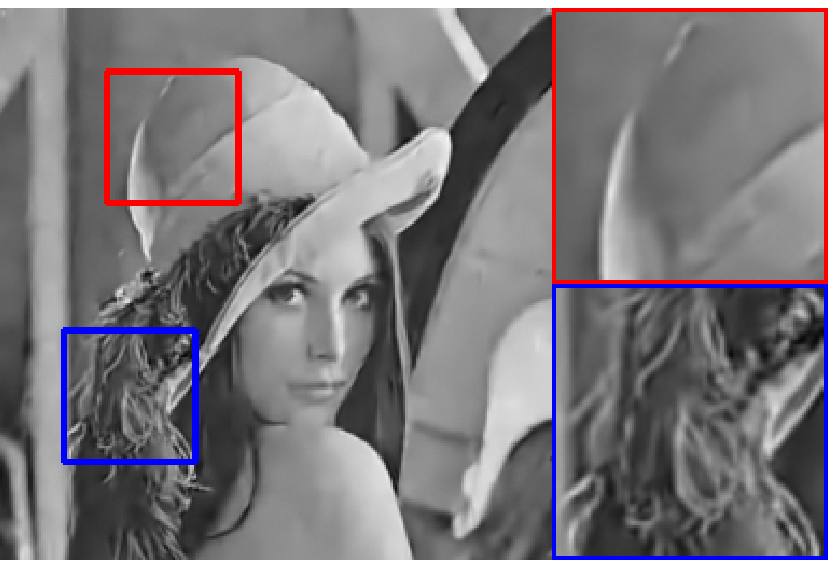}&
\includegraphics[height=\ih,width=\sw]{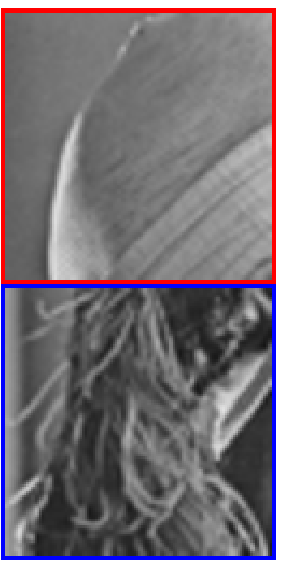}\\
$\sigma = 20$ & 31.60 dB (0.8960) & 31.41 dB (0.8917) & 31.82 dB (0.8998)\\

\includegraphics[height=\ih,width=\ih]{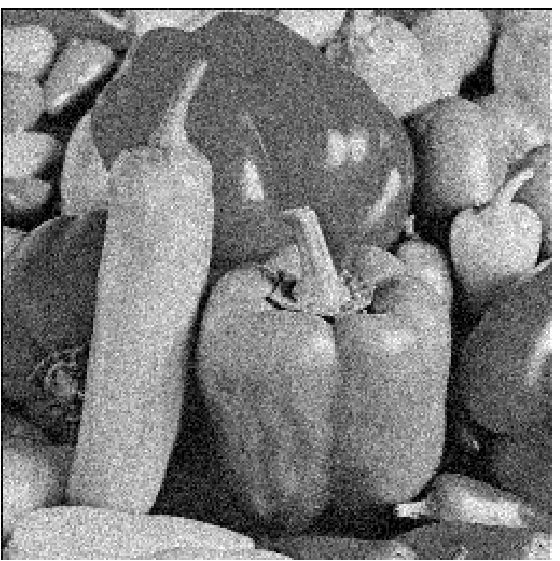}&
\includegraphics[height=\ih,width=\iw]{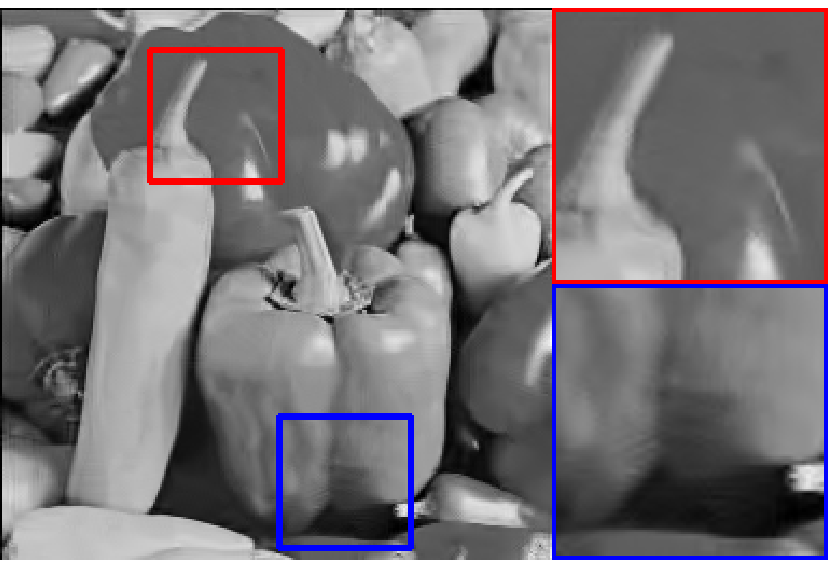}&
\includegraphics[height=\ih,width=\iw]{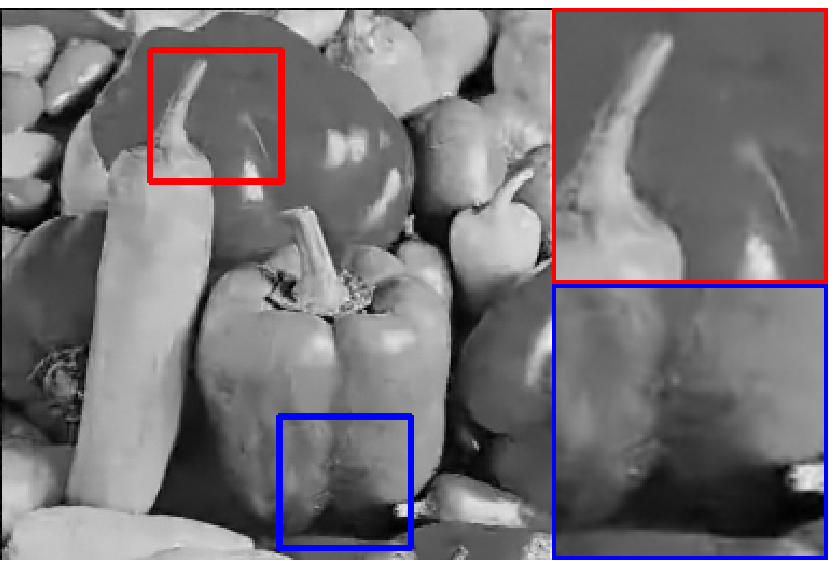}&
\includegraphics[height=\ih,width=\iw]{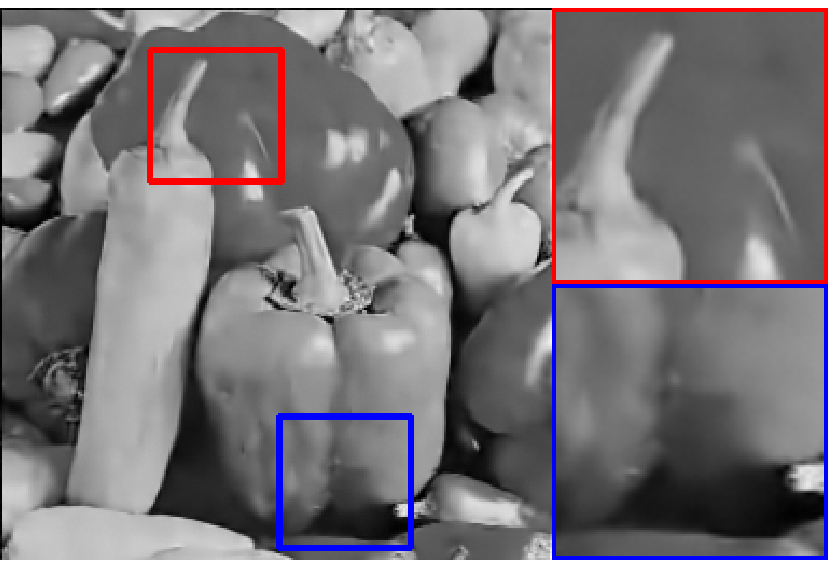}&
\includegraphics[height=\ih,width=\sw]{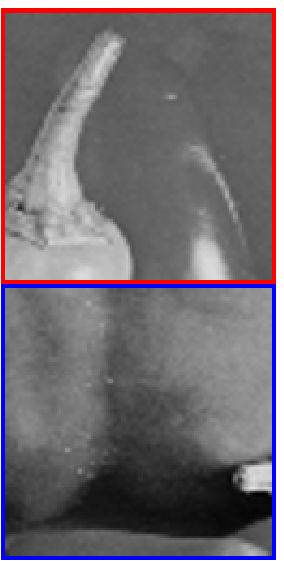}\\
$\sigma = 20$ & 31.14 dB (0.8843) & 31.12 dB (0.8859) & 31.44 dB (0.8926)\\

\includegraphics[height=\ih,width=\ih]{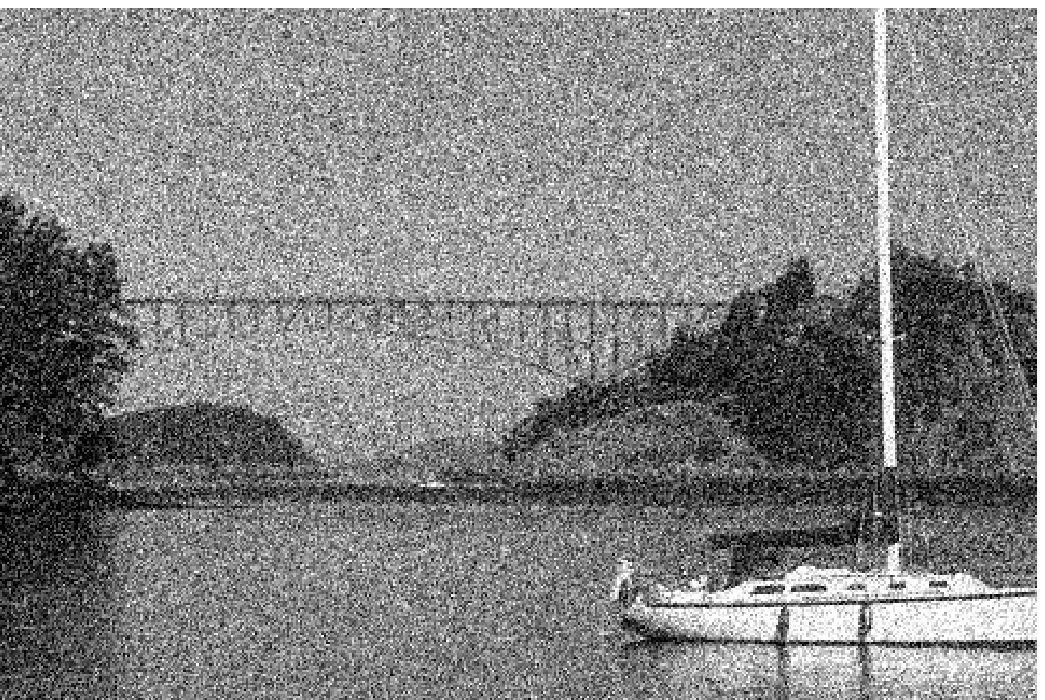}&
\includegraphics[height=\ih,width=\iw]{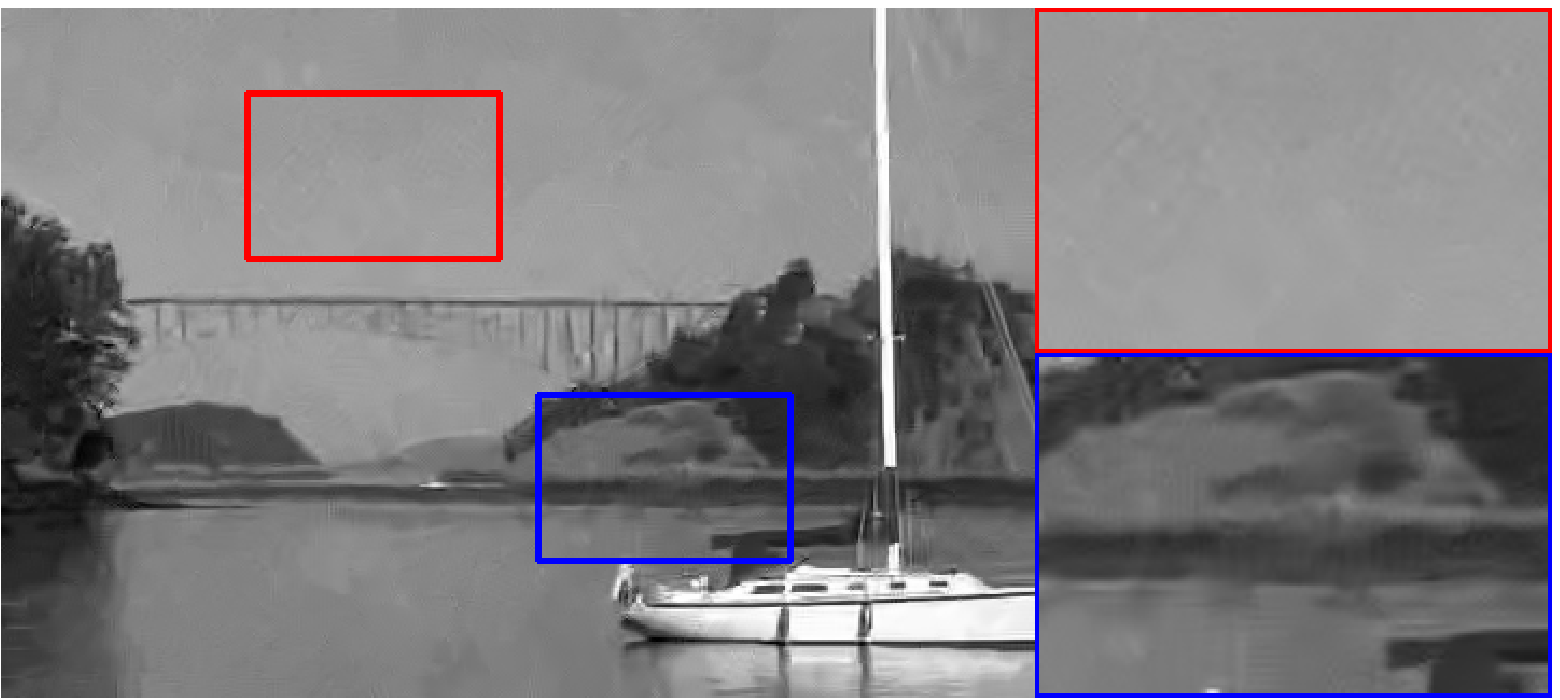}&
\includegraphics[height=\ih,width=\iw]{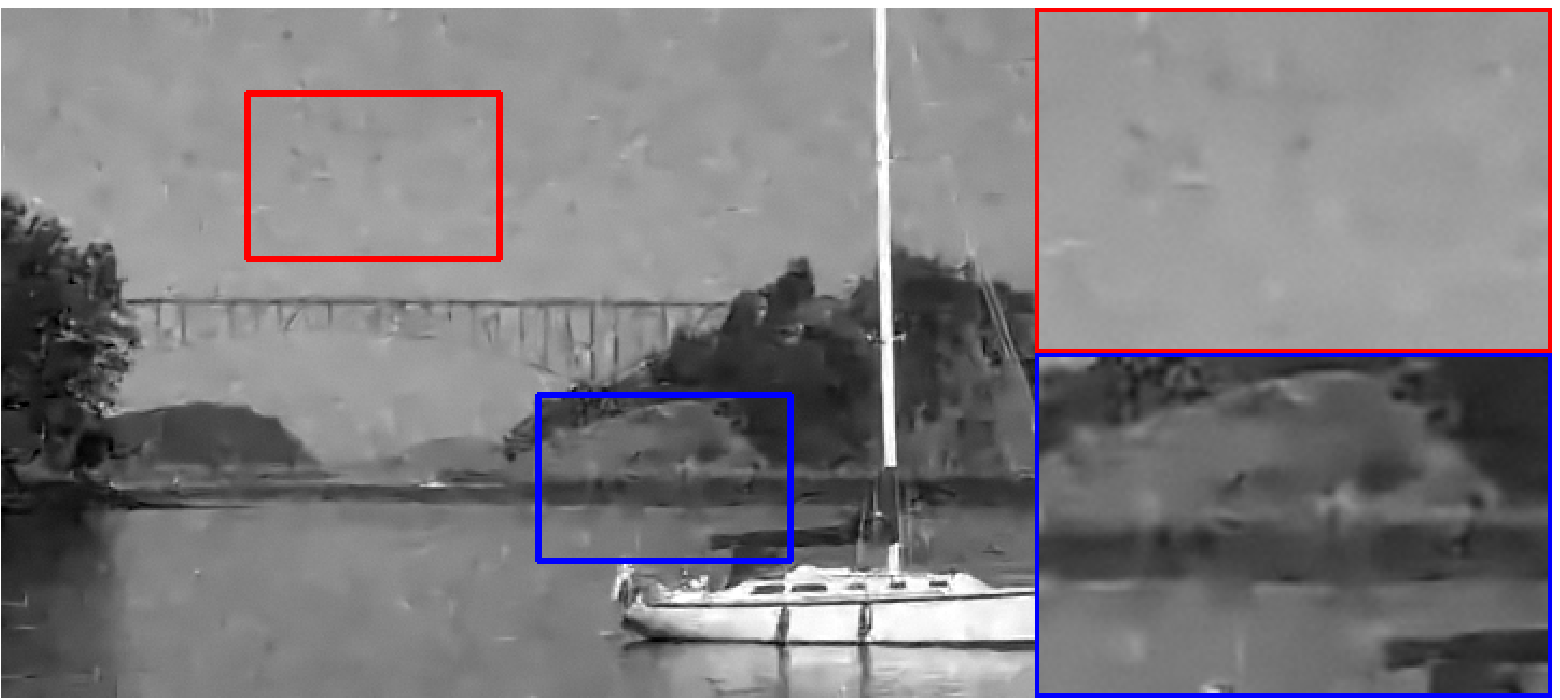}&
\includegraphics[height=\ih,width=\iw]{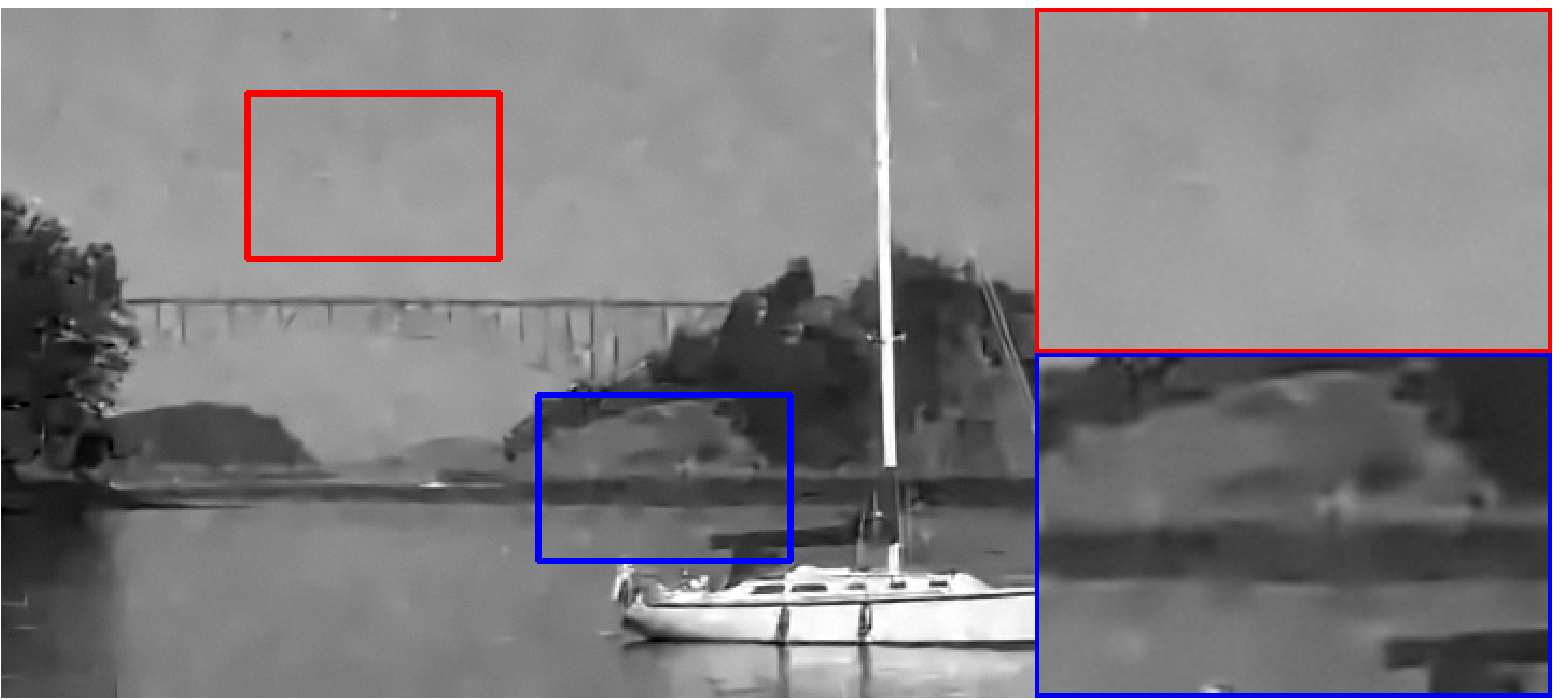}&
\includegraphics[height=\ih,width=\sw]{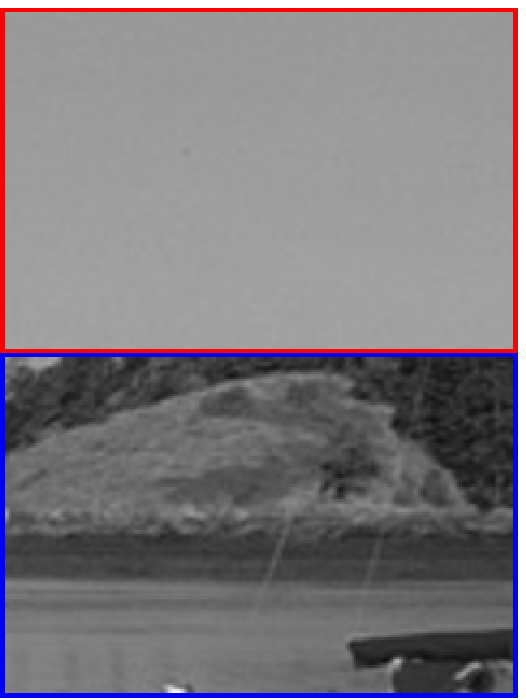}\\
$\sigma = 40$ & 28.78 dB (0.8196) & 28.69 dB (0.8103) & 28.90 dB (0.8270)\\

\includegraphics[height=\ih,width=\ih]{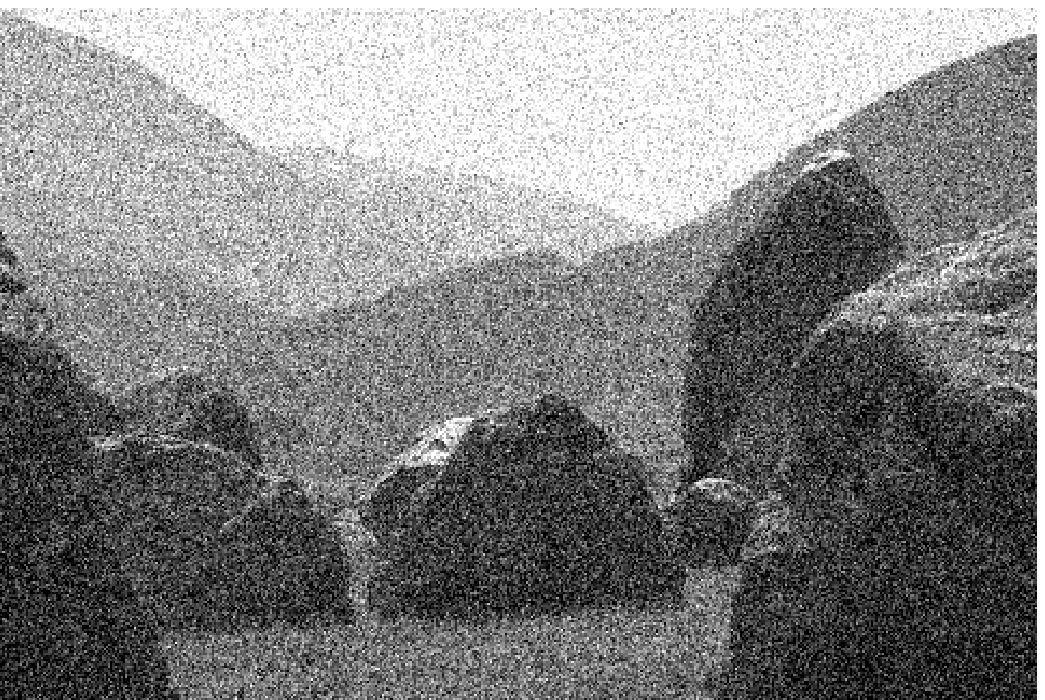}&
\includegraphics[height=\ih,width=\iw]{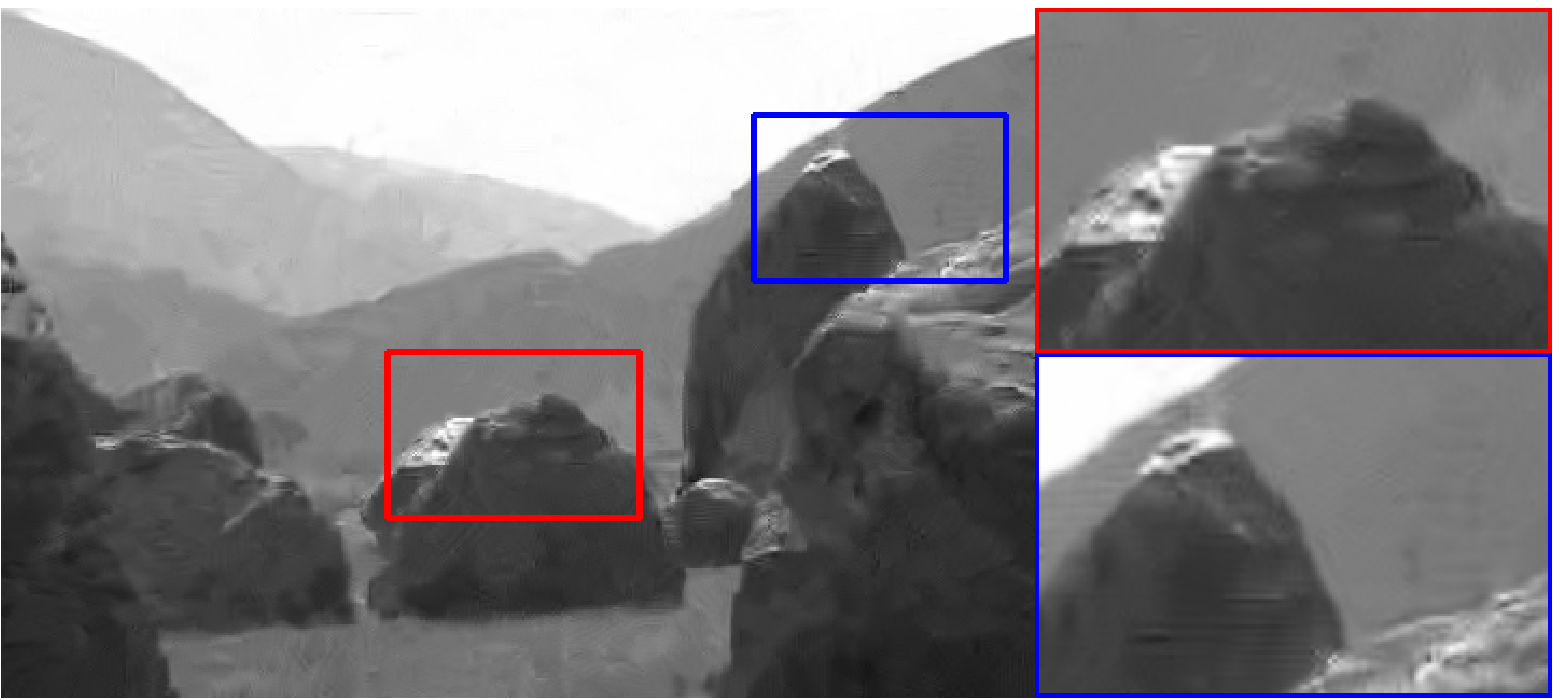}&
\includegraphics[height=\ih,width=\iw]{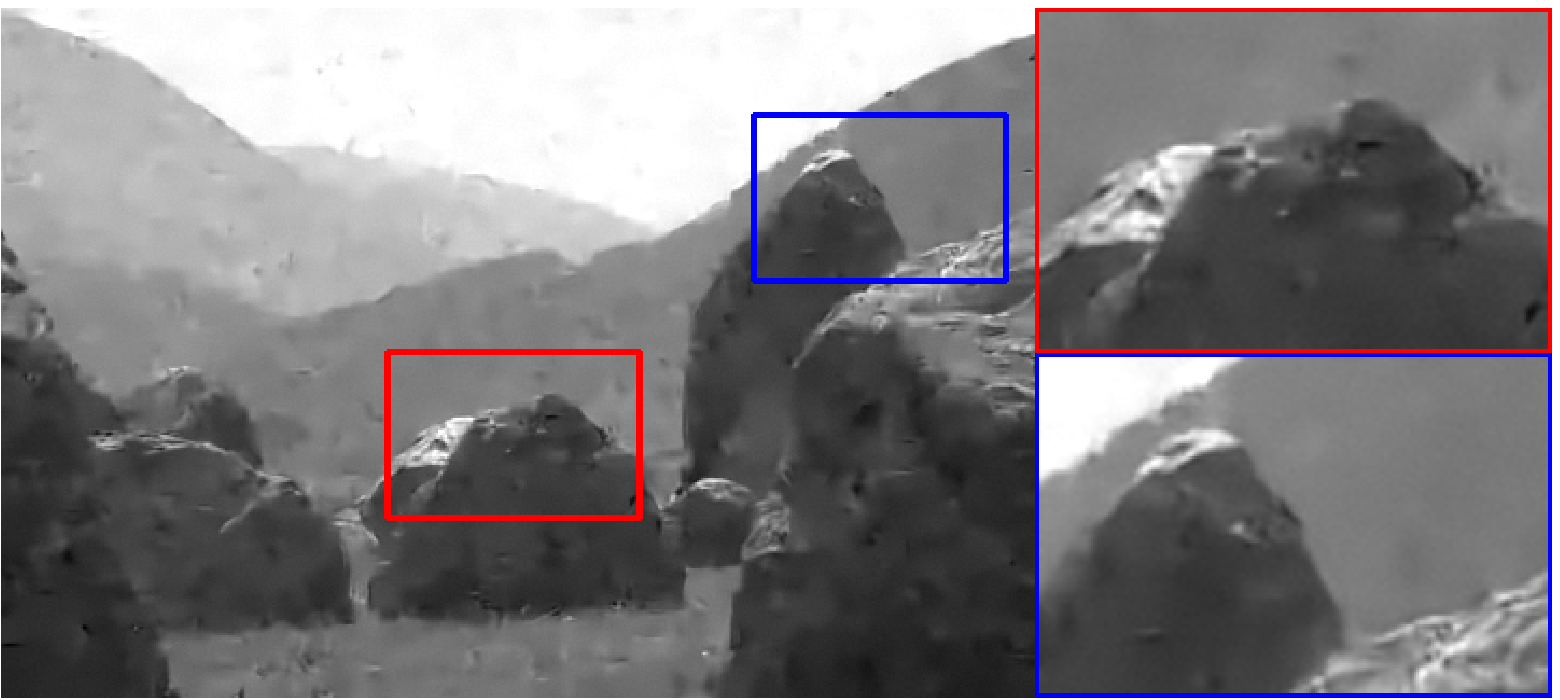}&
\includegraphics[height=\ih,width=\iw]{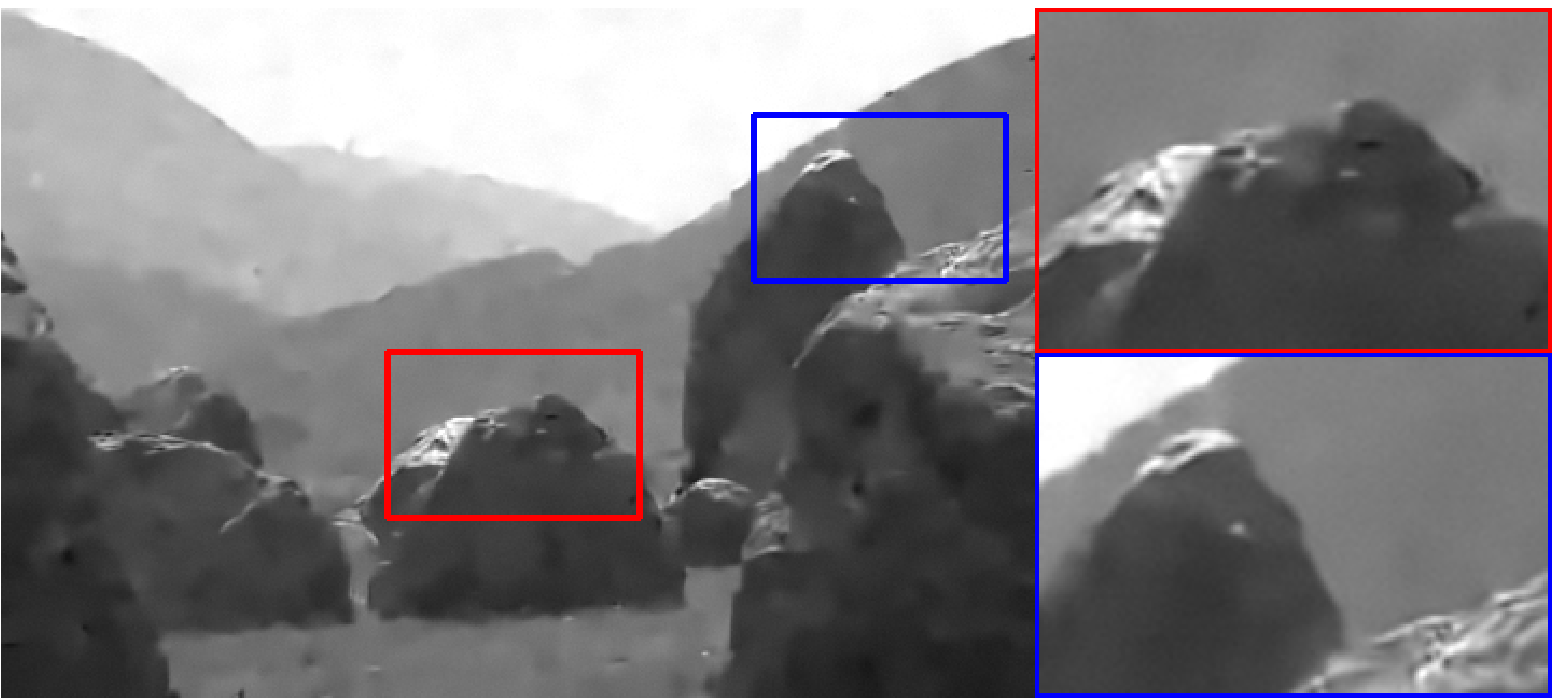}&
\includegraphics[height=\ih,width=\sw]{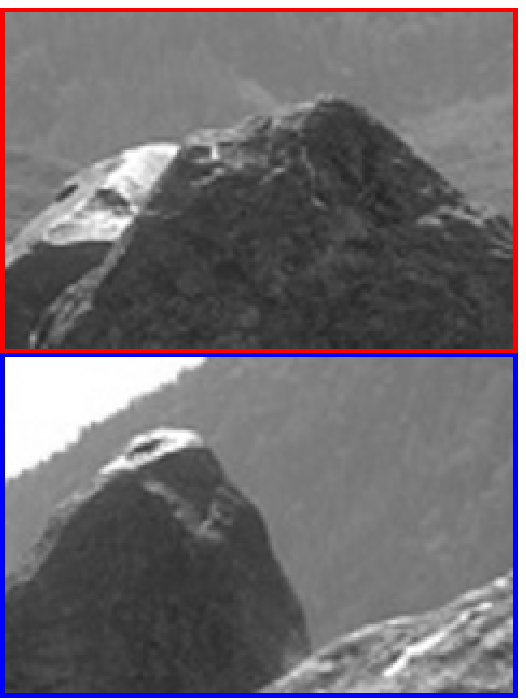}\\
$\sigma = 40$ & 29.43 dB (0.7597) & 29.45 dB (0.7555) & 29.70 dB (0.7652)\\

\includegraphics[height=\ih,width=\ih]{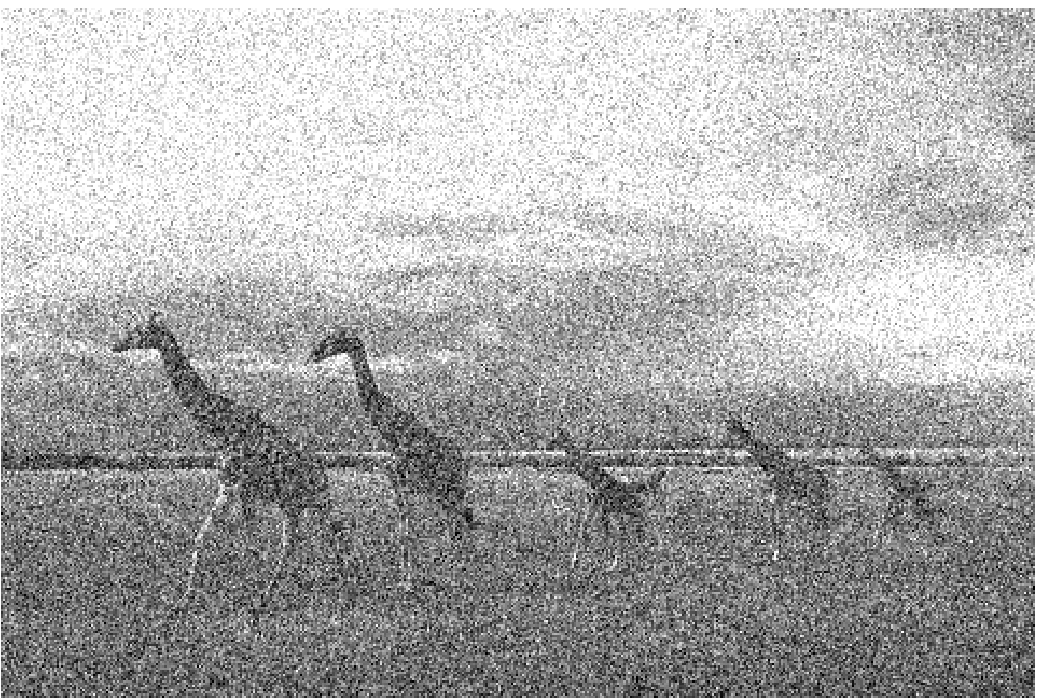}&
\includegraphics[height=\ih,width=\iw]{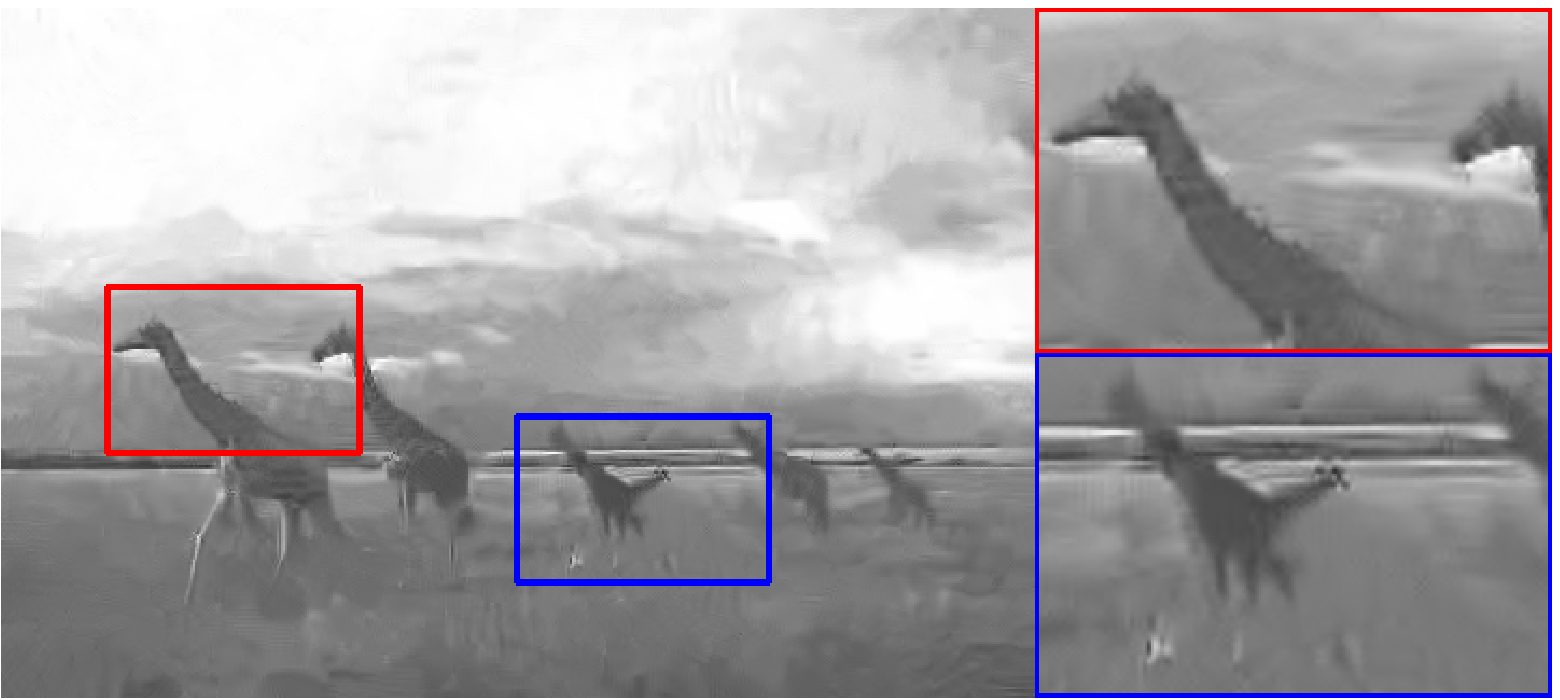}&
\includegraphics[height=\ih,width=\iw]{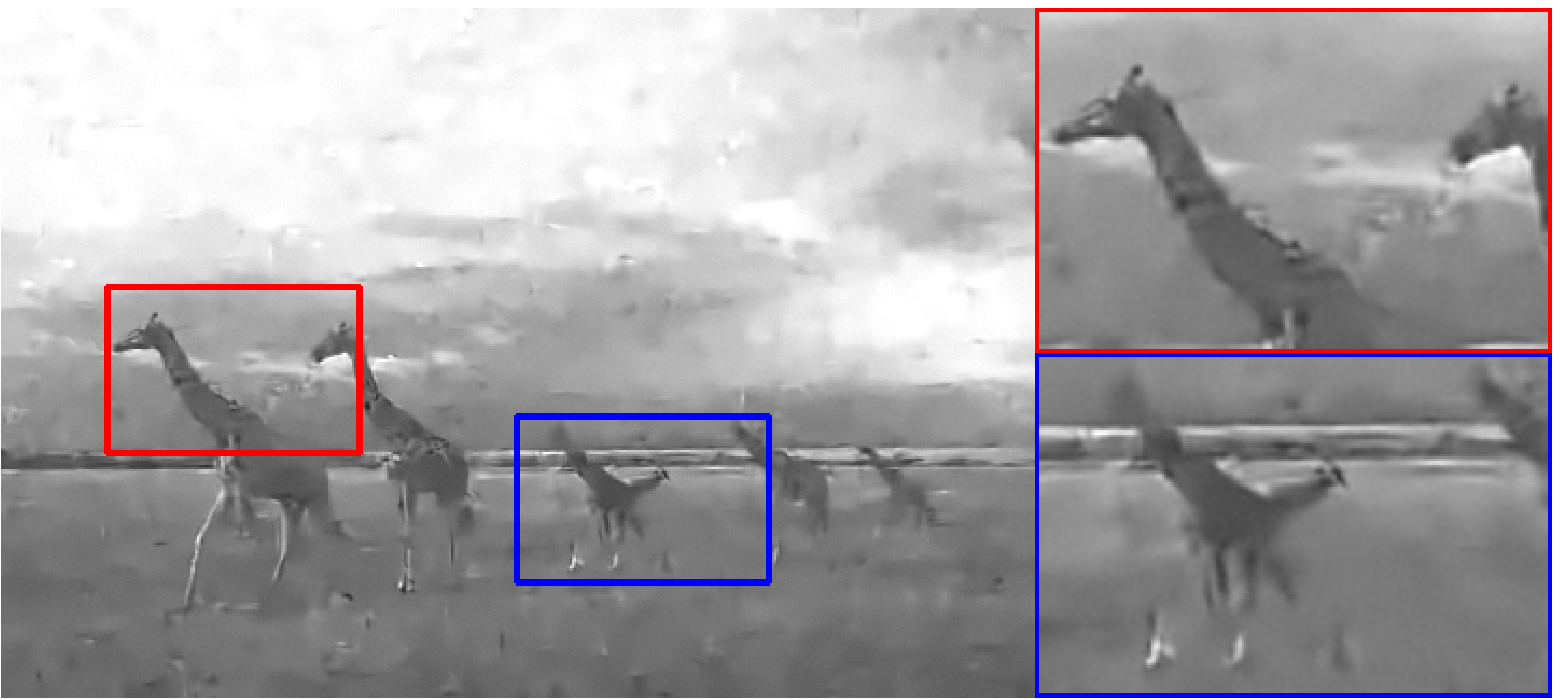}&
\includegraphics[height=\ih,width=\iw]{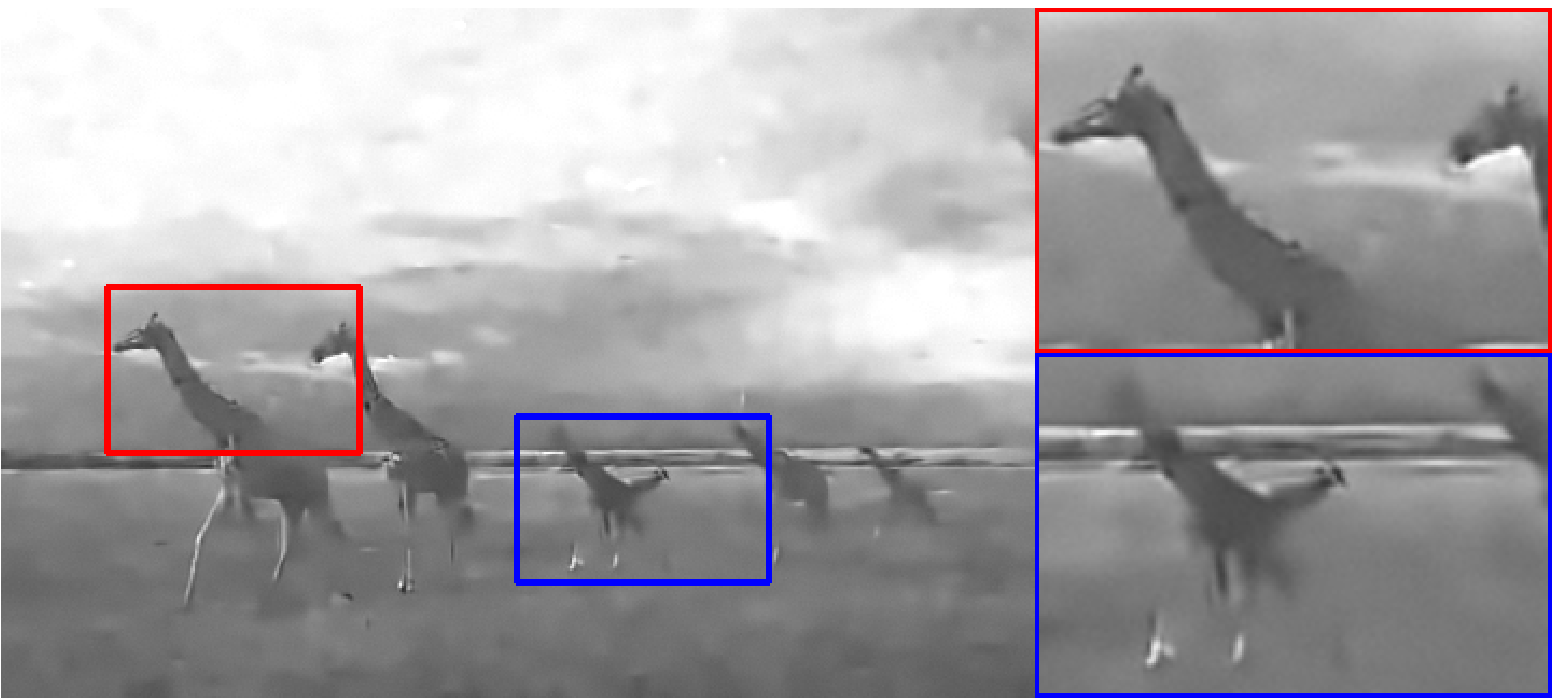}&
\includegraphics[height=\ih,width=\sw]{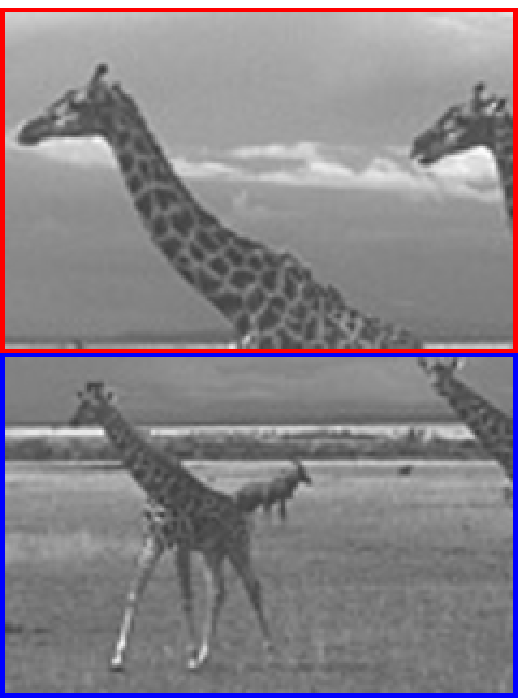}\\
$\sigma = 40$ & 29.80 dB (0.7687) & 29.93 dB (0.7655) & 30.21 dB (0.7751)\\
\end{tabular}
\caption{Single image denoising by using the denoised image for EM adaptation: Visual comparison and objective comparison (PSNR and SSIM in the parenthesis). The top three are standard images of size $256 \times 256$ while the bottom three are natural images of size $481 \times 321$.}
\label{fig:single image denoising for standard and natural images}
\end{figure*}

Figure \ref{fig:single image denoising for standard and natural images} shows the denoising results for three standard testing images and three natural testing images. In comparison to the competing methods, our proposed method yields the highest PSNR values. Magnified versions of the images are shown in \fref{fig:single image denoising for standard and natural images}.

In Table \ref{tbl:objective comparison for standard images of size 256x256}, we report PSNR for different noise variances for the standard images. Two key observations are noted:

First, comparing aGMM-EPLL with EPLL, the denoising results from aGMM-EPLL are consistently better than EPLL with an average gain of about 0.3 dB. This validates the usefulness of the adapted GMM through the proposed EM adaptation. For BM3D, we note that the prior model of BM3D (a non-linear transformation followed by a sparsity regularization) is fundamentally different from the EPLL model. Therefore, adapting a Gaussian mixture to a BM3D pre-filtered result does not necessarily improve the performance. In fact, aGMM-BM3D is better than BM3D at low noise, but worse at high noise.

Second, the quality of the image used for EM adaptation affects the final denoising performance. For example, using the ground-truth clean image for EM adaptation is much better than using the denoised images, such as the EPLL or BM3D denoised image. In some cases, aGMM-BM3D yields larger PSNR values than aGMM-EPLL due to the fact that the denoised image from BM3D is better than that from EPLL.

We validate our experiment by repeating 10 Monte Carlo trials. Consequently, we calculate the standard deviation across different noise realizations and conduct a $t$ test with the null hypothesis that aGMM-EPLL has equal performance to the original EPLL. If we reject the null hypothesis, then we conclude that aGMM-EPLL is statistically better than EPLL. Table \ref{PSNR_variability} shows the result of ``Peppers'' using 10 Monte Carlo trials. Except for the case when $\sigma = 80$, all \emph{p}-values are small, implying that the results by aGMM-EPLL are statistically better than those by EPLL.

In addition to comparing EPLL and BM3D, we also compare the performance of the proposed method when there is no generic prior. That is, we train a GMM purely from the noisy image. In this case, we collect the noisy patches from the image and apply EM algorithm to learn a GMM. The results, denoted as \emph{GMM-noisy}, are shown in Table~\ref{table:naive gmm}. The performance is undesirable, as we expected. For comparison, we adapt the generic GMM to the noisy image with Equations (25) and (26) and set $\widetilde{\sigma}^2 = \sigma^2$. In this case, the performance is improved significantly by the adaptation process, confirming the effectiveness of the method.

We compare the result with \cite{Lu_Lin_Jin_Yang_Wang_2015} in the last two columns of Table~\ref{table:naive gmm}. The result indicates that \cite{Lu_Lin_Jin_Yang_Wang_2015} performs consistently better than EPLL but slightly worse than the proposed aGMM-EPLL. One reason is that in our method, the pre-filtering has significantly reduced the noise level for better training.


\begin{table}[h]
\centering
\footnotesize
\[
\renewcommand{\arraystretch}{1.01}
\def\C{@{\hspace*{3.2ex}}c@{\hspace*{3.2ex}}}
\begin{array}{cc|cc||cc||c}
\hline
& & \text{BM3D} & \text{aGMM} & \text{EPLL} & \text{aGMM} & \text{aGMM}\\
& &             & \text{-BM3D}&             & \text{-EPLL}& \text{-clean} \\
\hline
\hline
\multirow{5}{*}{\text{Airplane}}
&	\sigma=20 	&	30.44	&	\textbf{30.77}	&	30.57	&	\textbf{30.87}	&	31.28	\\
& 	\sigma=40 	& 	26.45	& 	\textbf{27.09}	& 	27.00	& 	\textbf{27.16}	& 	27.48	\\
& 	\sigma=60 	& 	\textbf{25.15}	& 	25.09	& 	25.14	& 	\textbf{25.24}	& 	25.50	\\
& 	\sigma=80 	& 	\textbf{23.85}	& 	23.72	& 	23.74	& 	\textbf{23.83}	& 	24.00	\\
& 	\sigma=100 	& 	\textbf{22.82}	& 	22.60	& 	22.61	& 	\textbf{22.66}	& 	22.80	\\
\hline
\multirow{5}{*}{\text{Boat}}
&	\sigma=20 	&	29.69	&	\textbf{29.90}	&	29.83	&	\textbf{30.00}	&	30.39	\\
& 	\sigma=40 	& 	26.09	& 	\textbf{26.57}	& 	26.46	& 	\textbf{26.60}	& 	26.86	\\
& 	\sigma=60 	& 	24.58	& 	\textbf{24.65}	& 	24.69	& 	\textbf{24.77}	& 	25.01	\\
& 	\sigma=80 	& 	\textbf{23.40}	& 	23.36	& 	23.41	& 	\textbf{23.46}	& 	23.69	\\
& 	\sigma=100 	& 	\textbf{22.64}	& 	22.56	& 	22.58	& 	\textbf{22.61}	& 	22.76	\\
\hline
\multirow{5}{*}{\text{Cameraman}}
&	\sigma=20 	&	30.28	&	\textbf{30.33}	&	30.21	&	\textbf{30.38}	&	31.09	\\
& 	\sigma=40 	& 	26.78	& 	\textbf{27.29}	& 	26.96	& 	\textbf{27.25}	& 	27.76	\\
& 	\sigma=60 	& 	25.35	& 	\textbf{25.42}	& 	25.24	& 	\textbf{25.52}	& 	26.07	\\
& 	\sigma=80 	& 	\textbf{24.05}	& 	24.04	& 	23.90	& 	\textbf{24.14}	& 	24.66	\\
& 	\sigma=100 	& 	\textbf{23.05}	& 	22.88	& 	22.79	& 	\textbf{22.94}	& 	23.41	\\

\hline
\multirow{5}{*}{\text{House}}
&	\sigma=20 	&	33.67	&	\textbf{33.81}	&	33.03	&	\textbf{33.63}	&	34.33	\\
& 	\sigma=40 	& 	30.49	& 	\textbf{30.85}	& 	29.94	& 	\textbf{30.64}	& 	31.31	\\
& 	\sigma=60 	& 	\textbf{28.88}	& 	28.73	& 	27.97	& 	\textbf{28.57}	& 	29.19	\\
& 	\sigma=80 	& 	\textbf{27.12}	& 	26.95	& 	26.34	& 	\textbf{26.87}	& 	27.28	\\
& 	\sigma=100 	& 	\textbf{25.92}	& 	25.70	& 	25.33	& 	\textbf{25.67}	& 	26.01	\\
\hline
\multirow{5}{*}{\text{Lena}}
&	\sigma=20 	&	31.60	&	\textbf{31.76}	&	31.41	&	\textbf{31.82}	&	32.37	\\
& 	\sigma=40 	& 	27.83	& 	\textbf{28.18}	& 	27.98	& 	\textbf{28.25}	& 	28.62	\\
& 	\sigma=60 	& 	\textbf{26.36}	& 	26.16	& 	26.03	& 	\textbf{26.23}	& 	26.51	\\
& 	\sigma=80 	& 	\textbf{25.05}	& 	24.85	& 	24.70	& 	\textbf{24.91}	& 	25.12	\\
& 	\sigma=100 	& 	\textbf{23.88}	& 	23.76	& 	23.58	& 	\textbf{23.79}	& 	23.96	\\
\hline
\multirow{5}{*}{\text{Peppers}}
&	\sigma=20 	&	31.14	&	\textbf{31.40}	&	31.12	&	\textbf{31.44}	&	32.04	\\
& 	\sigma=40 	& 	27.42	& 	\textbf{28.00}	& 	27.70	& 	\textbf{28.03}	& 	28.43	\\
& 	\sigma=60 	& 	25.87	& 	\textbf{25.98}	& 	25.70	& 	\textbf{26.06}	& 	26.39	\\
& 	\sigma=80 	& 	24.43	& 	\textbf{24.56}	& 	24.25	& 	\textbf{24.64}	& 	24.92	\\
& 	\sigma=100 	& 	23.28	& 	\textbf{23.30}	& 	23.05	& 	\textbf{23.39}	& 	23.61	\\
\hline
 \text{Average}  &     & 26.59	&	\textbf{26.68}	&	26.44	&	\textbf{26.71}	&	27.09 \\
\end{array}
\]
\caption{Results for images of size $256 \times 256$. The PSNR value for each noise level is averaged over 8 independent trials to reduce the bias due to a particular noise realization.}
\label{tbl:objective comparison for standard images of size 256x256}
\end{table}

\begin{table}[h]
\centering
\footnotesize
\[
\renewcommand{\arraystretch}{1.01}
\def\C{@{\hspace*{3.2ex}}c@{\hspace*{3.2ex}}}
\begin{array}{c||cc|cc|c}
\hline
 & \text{EPLL} & \text{std.} & \text{aGMM-EPLL} &  \text{std.} & \text{\emph{p}-value}\\
\hline
\sigma = 20 & 29.51	& 0.11 & 29.84 & 0.12 & 1.0\mathrm{e}{-5} \\
\sigma = 40 & 25.82	& 0.10 & 26.04 & 0.10 & 1.3\mathrm{e}{-3}\\
\sigma = 60 & 23.65	& 0.12 & 23.80 & 0.10 & 1.5\mathrm{e}{-2}\\
\sigma = 80 & 22.34	& 0.14 & 22.46 & 0.16 & 1.1\mathrm{e}{-1}\\
\sigma = 100& 21.15	& 0.13 & 21.31 & 0.16 & 3.6\mathrm{e}{-2}\\
\hline
\end{array}
\]
\caption{PSNR variability for different denoising methods under different noise levels. The testing image is ``Peppers'' of size $128 \times 128$. The PSNR values and standard deviations are computed from 10 Monte-Carlo trials.}
\label{PSNR_variability}
\end{table}

\begin{table}[h]
\centering
\footnotesize
\[
\renewcommand{\arraystretch}{1.01}
\def\C{@{\hspace*{3.2ex}}c@{\hspace*{3.2ex}}}
\begin{array}{c||cc|ccc}
\hline
& \text{GMM} & \text{aGMM} & \text{EPLL} & \text{\cite{Lu_Lin_Jin_Yang_Wang_2015}}  & \text{aGMM} \\
& \text{-noisy} & \text{-noisy} &        &                                          &\text{-EPLL} \\
\hline
\hline
\sigma = 20 & 22.31 & 29.01	&	29.53	&	29.55	&	29.84\\
\sigma = 40 & 21.52 & 25.15	& 	25.85	& 	25.90	& 	26.03\\
\sigma = 60 & 20.40 & 23.01	& 	23.71	& 	23.80	& 	23.83\\
\sigma = 80 & 19.48 & 21.53	& 	22.22	& 	22.30	& 	22.36\\
\sigma = 100& 18.81 & 20.49	& 	21.14	& 	21.22	& 	21.31\\
\hline
\end{array}
\]
\caption{Comparison of the proposed method with different initialization and adaptation. GMM-noisy: Directly learn the GMM from the noisy image without adaptation. aGMM-noisy: Learn the GMM from the noisy image with adaptation. \cite{Lu_Lin_Jin_Yang_Wang_2015} and aGMM-EPLL: Learn GMMs from the pre-filtered results without and with compensating the remaining noise in the pre-filtered image, respectively. The test image is Peppers of size $128 \times 128$. PSNR is averaged over 10 Monte-Carlo trials.}
\label{table:naive gmm}
\vspace{-4ex}
\end{table}

In order to visually compare the improvement of the adapted GMM over the generic GMM, we randomly sample 100 patches from each GMM and show the results in Figure \ref{gmm patch}. The patches we display are the raw un-normalized patches. The gray scale value reflects the actual magnitude of a patch, and the occurrence of the patches indicates the likelihood of drawing that patch. Therefore, a significant portion of the patches are smooth as the GMM has peaks at smooth patches. The results show that the adapted GMM does have improved patch quality compared to the generic one.

\begin{figure}[!]
\centering
\def\iw{0.45\linewidth}
\begin{tabular}{ccc}
\includegraphics[width=\iw]{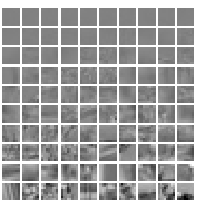}&
\includegraphics[width=\iw]{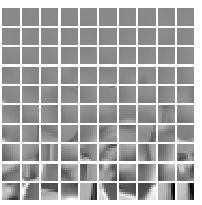}\\
(a) generic GMM & (b) adapted GMM
\end{tabular}
\caption{Patches sampled from generic GMM and adapted GMM: The $100$ samples in (a) and (b) are randomly sampled and generated from the generic GMM and the adapted GMM, respectively. In this experiment, the generic GMM is adapted to a clean image, Peppers, to create the adapted GMM. All samples are of size $8 \times 8$.}
\label{gmm patch}
\end{figure}



\subsection{External Image Denoising}
In this subsection, we evaluate the denoising performance when an \emph{example} image is available for EM adaptation. An example image refers to a clean image and is relevant to the noisy image of interest. In \cite{Luo_Chan_Nguyen_2014, Luo_Chan_Nguyen_2015}, it is shown that obtaining reference images is feasible in some scenarios such as text images and face images. We consider the following three scenarios for our experiments:

\begin{enumerate}
\item Flower image denoising: We use the 102 flowers dataset from \cite{Nilsback_Zisserman_2008}, which consists of 102 different categories of flowers. We randomly pick one category and then sample two flower images: one as the testing image with additive i.i.d. Gaussian noise, and the other as the example image for the EM adaptation.

\item Face image denoising: We use the FEI face dataset from \cite{Thomaz_Giraldi_2010}, which consists of 100 aligned and frontal face images of size $260 \times 360$. We randomly pick one face image as the image of interest. We then randomly sample another image from the dataset and treat it as the example image for our EM adaptation.

\item Text image denoising: To prepare for this scenario, we randomly crop a $200 \times 200$ region from a document and add noise to it. We then crop another $200 \times 200$ region from a very different document and use it as the example image.
\end{enumerate}

In Figure \ref{fig:external image denoising: visual comparison}, we show the denoising results for the three different scenarios. As is shown, the example images in the second column are similar but differ from the testing images. We compare the three denoising methods. The major difference lies in how the default GMM is adapted: In EPLL, there is no EM adaptation, i.e., the default generic GMM is used for denoising. In aGMM-example, the default GMM is adapted to the example image while in aGMM-clean, the default GMM is adapted to the ground truth image. As observed, the oracle aGMM-clean yields the best denoising performance. aGMM-example outperforms the benchmark EPLL (generic GMM) denoising algorithm both visually and objectively. For example, on average, it is 0.28 dB better in the Flower scenario, 0.78 dB better in the Face scenario, and 1.57 dB better in the Text scenario.

In addition to adapting the generic GMM to the example image, we could also apply the EM algorithm on the example image to learn a new GMM. However, we argue that the new GMM would lead to over-fitting and, hence, poor performance for denoising. In Table \ref{comparison_of_different_GMMs} we show the PSNR results when we use different GMMs for denoising. As is observed, GMM-example, which learns a new GMM from the example image, is even worse than EPLL whose GMM is learned from a generic database. In contrast, our proposed aGMM-example, which adapts the generic GMM to the example image, gives the best performance consistently.

\begin{table}[h]
\centering
\footnotesize
\[
\renewcommand{\arraystretch}{1.01}
\def\C{@{\hspace*{3.2ex}}c@{\hspace*{3.2ex}}}
\begin{array}{c|ccc}
\hline
 & \text{EPLL} & \text{GMM} &  \text{aGMM}\\
 &             & \text{-example}& \text{-example}\\
\hline
\sigma = 20 & 33.09 \; (0.9556) & 32.07 \;(0.9489) & 33.39 \;(0.9604) \\
\sigma = 40 & 28.97 \; (0.9000) & 28.40 \;(0.8951) & 29.32 \;(0.9078) \\
\sigma = 60 & 26.97 \; (0.8485) & 26.55 \;(0.8447) & 27.24 \;(0.8579) \\
\sigma = 80 & 25.24 \; (0.8108) & 25.10 \;(0.8088) & 25.56 \;(0.8221) \\
\sigma = 100 & 24.27 \; (0.7790) & 24.16 \;(0.7755) & 24.52 \;(0.7886) \\
\hline
\end{array}
\]
\caption{External image denoising: PSNR and SSIM in the parenthesis. Two flower images are sampled from \cite{Nilsback_Zisserman_2008} with one being the testing image and the other being the example image. EPLL uses the generic GMM, GMM-example applies EM algorithm to the example image, and aGMM-example applies adaptation from the generic GMM to the example image.}
\label{comparison_of_different_GMMs}
\vspace{-4ex}
\end{table}

\begin{figure*}[!]
\centering
\def\ih{0.14\linewidth}
\def\iw{0.18\linewidth}
\def\sw{0.06\linewidth}
\begin{tabular}{cccccc}
noisy image & example image & EPLL & aGMM & aGMM & ground\\
 & & & -example & -clean & -truth\\
\includegraphics[height=\ih,width=\ih]{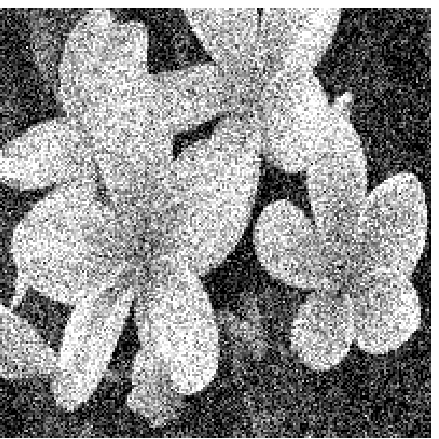}&
\includegraphics[height=\ih,width=\ih]{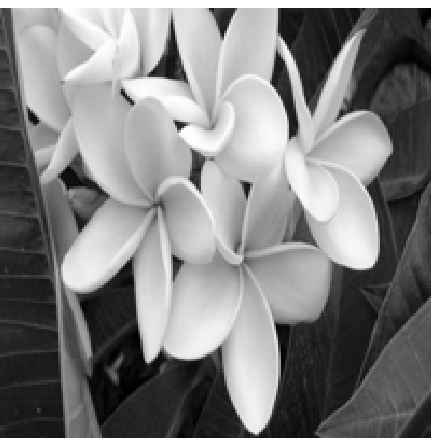}&
\includegraphics[height=\ih,width=\iw]{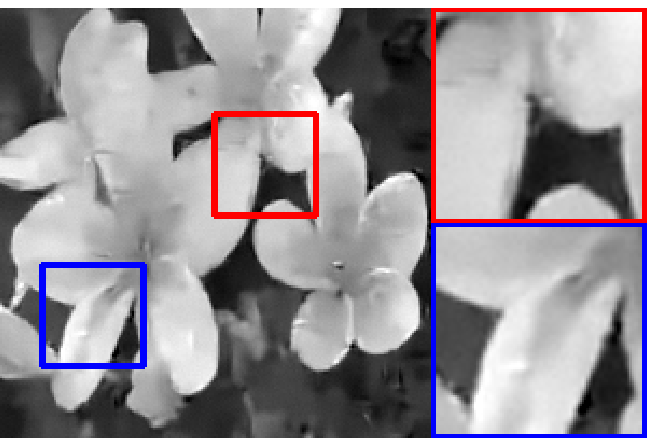}&
\includegraphics[height=\ih,width=\iw]{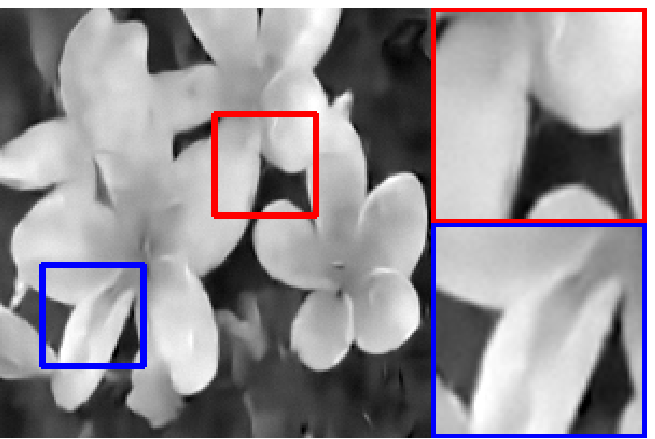}&
\includegraphics[height=\ih,width=\iw]{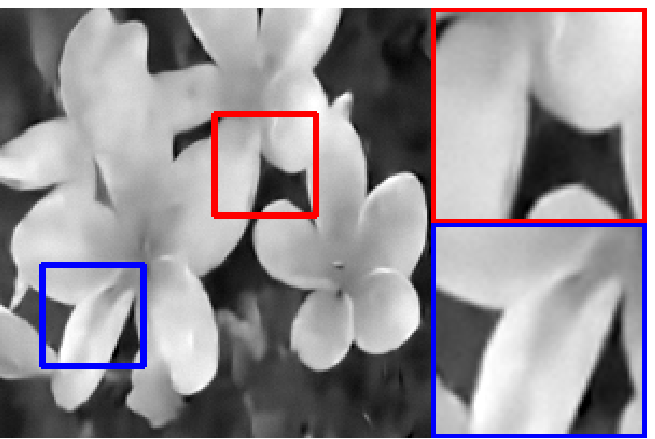}&
\includegraphics[height=\ih,width=\sw]{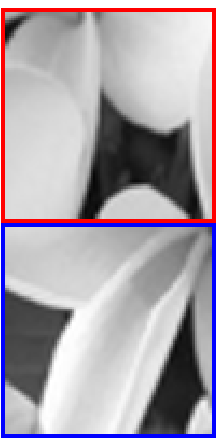}\\
$\sigma = 50$ &  & 26.90 dB (0.7918) & 27.28 dB (0.8051) & 27.84 dB (0.8181) &\\

\includegraphics[height=\ih,width=\ih]{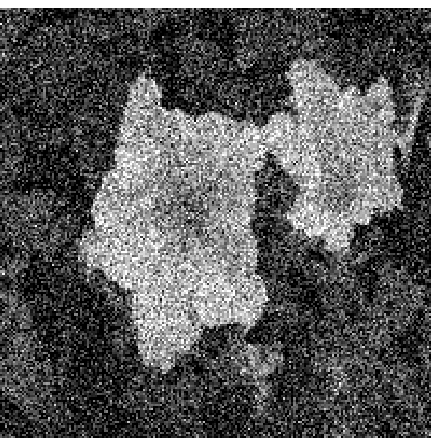}&
\includegraphics[height=\ih,width=\ih]{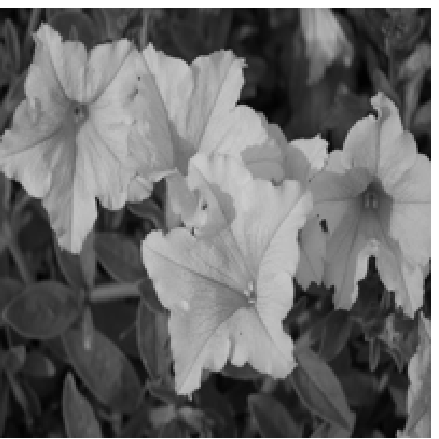}&
\includegraphics[height=\ih,width=\iw]{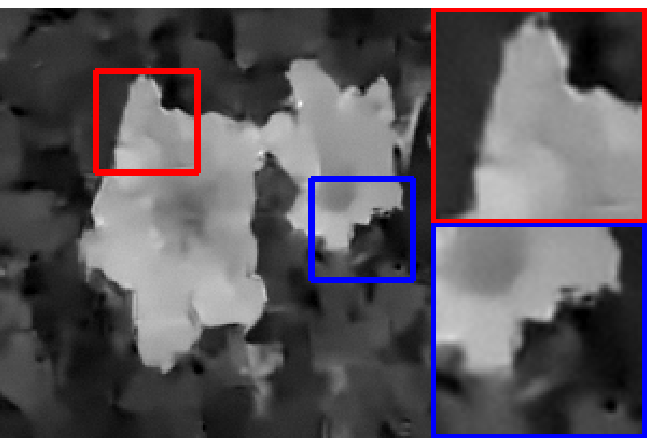}&
\includegraphics[height=\ih,width=\iw]{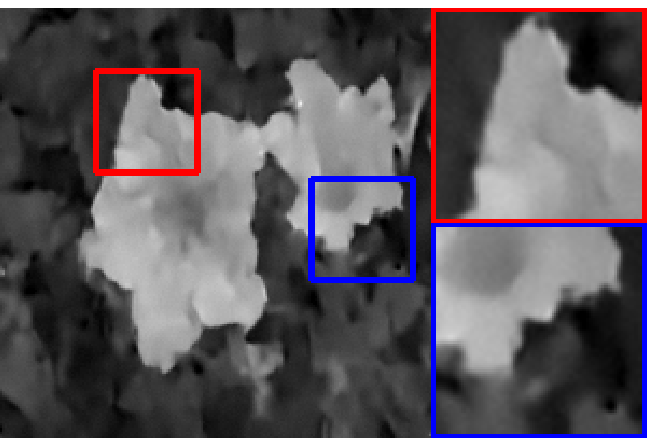}&
\includegraphics[height=\ih,width=\iw]{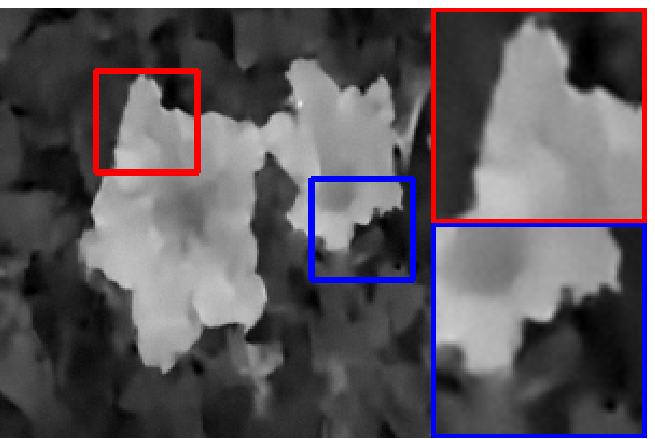}&
\includegraphics[height=\ih,width=\sw]{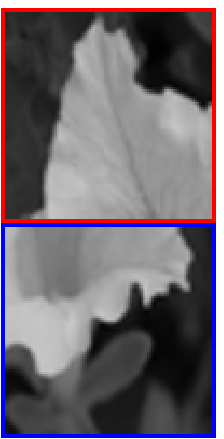}\\
$\sigma = 50$ &  & 27.49 dB (0.7428) & 27.68 dB (0.7507) & 28.06 dB (0.7613) &\\

\includegraphics[height=\ih,width=\ih]{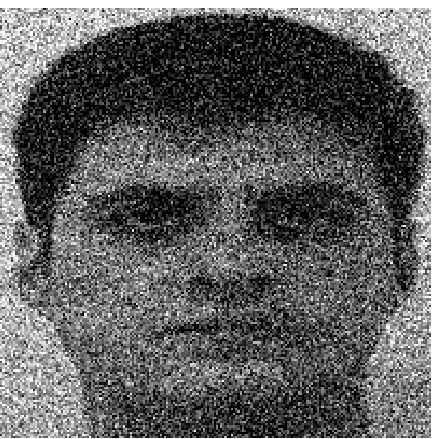}&
\includegraphics[height=\ih,width=\ih]{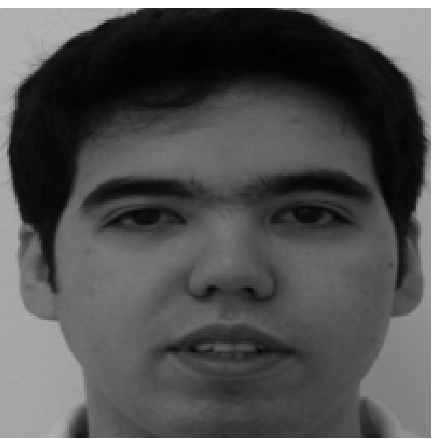}&
\includegraphics[height=\ih,width=\iw]{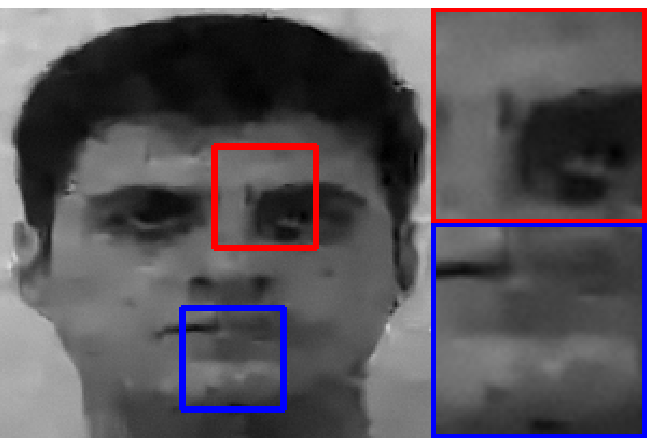}&
\includegraphics[height=\ih,width=\iw]{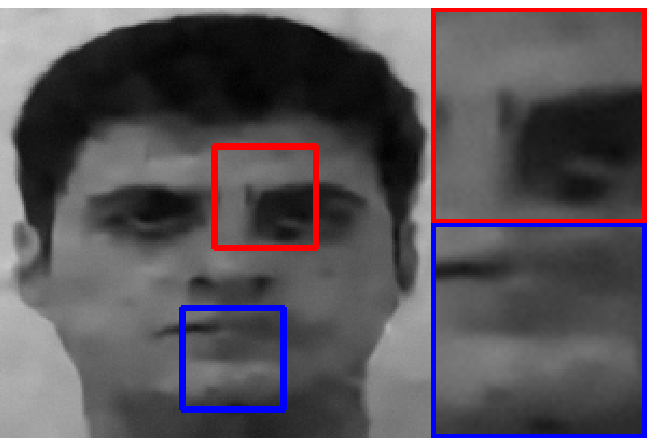}&
\includegraphics[height=\ih,width=\iw]{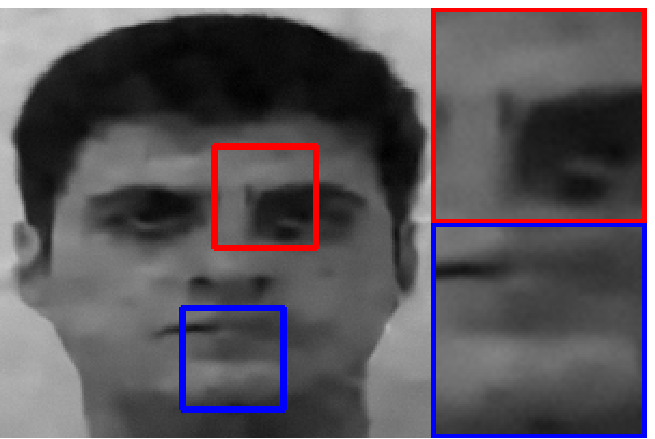}&
\includegraphics[height=\ih,width=\sw]{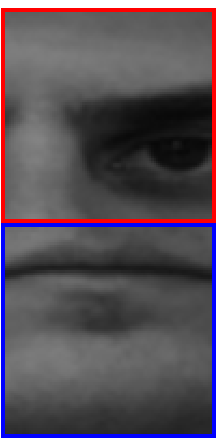}\\
$\sigma = 50$ &  & 29.79 dB (0.8414) & 30.53 dB (0.8611) & 30.68 dB (0.8630) &\\

\includegraphics[height=\ih,width=\ih]{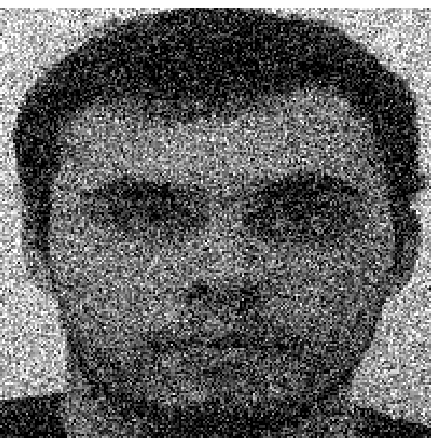}&
\includegraphics[height=\ih,width=\ih]{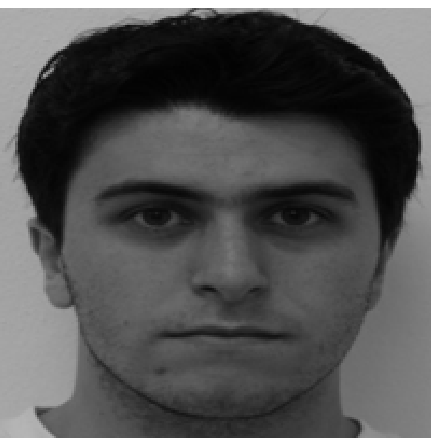}&
\includegraphics[height=\ih,width=\iw]{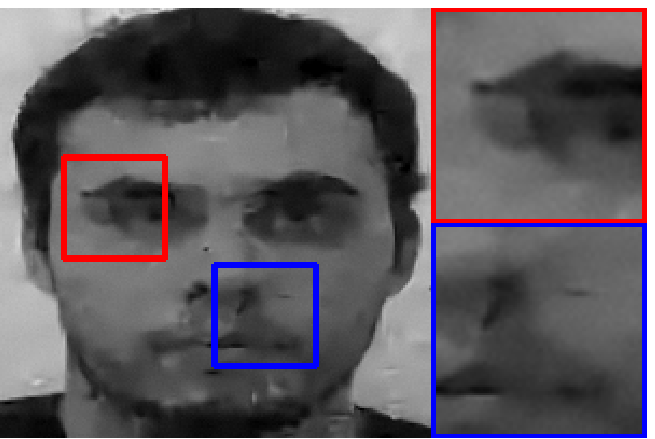}&
\includegraphics[height=\ih,width=\iw]{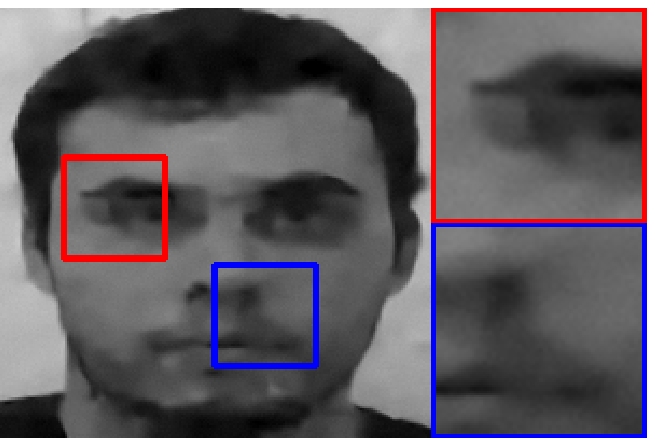}&
\includegraphics[height=\ih,width=\iw]{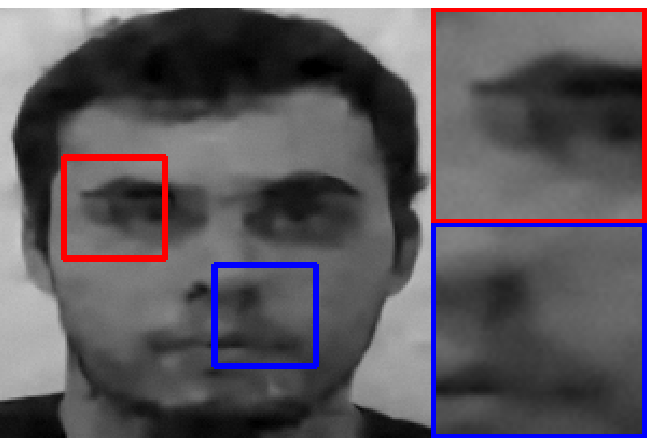}&
\includegraphics[height=\ih,width=\sw]{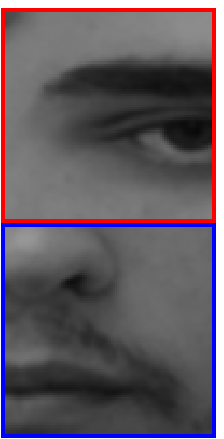}\\
$\sigma = 50$ &  & 29.44 dB (0.8233) & 30.26 dB (0.8513) & 30.52 dB (0.8528) &\\

\includegraphics[height=\ih,width=\ih]{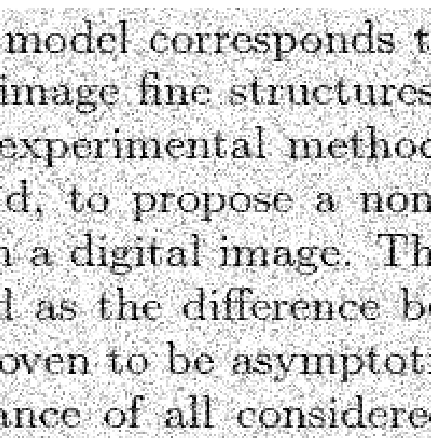}&
\includegraphics[height=\ih,width=\ih]{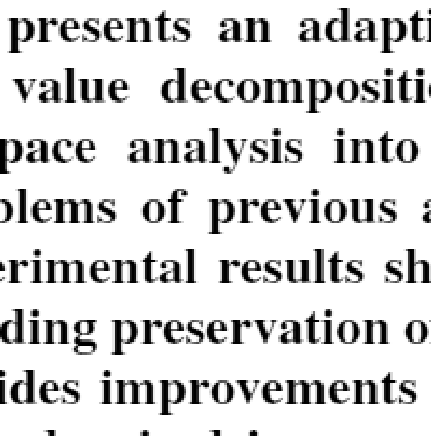}&
\includegraphics[height=\ih,width=\iw]{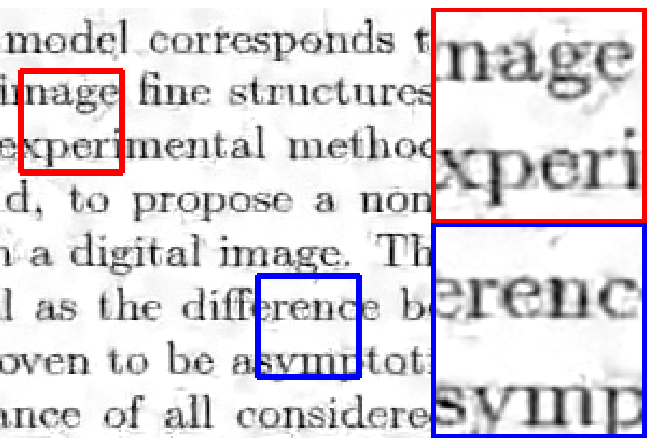}&
\includegraphics[height=\ih,width=\iw]{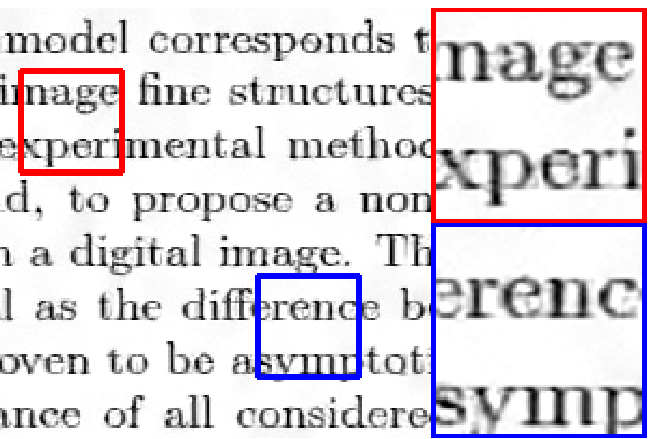}&
\includegraphics[height=\ih,width=\iw]{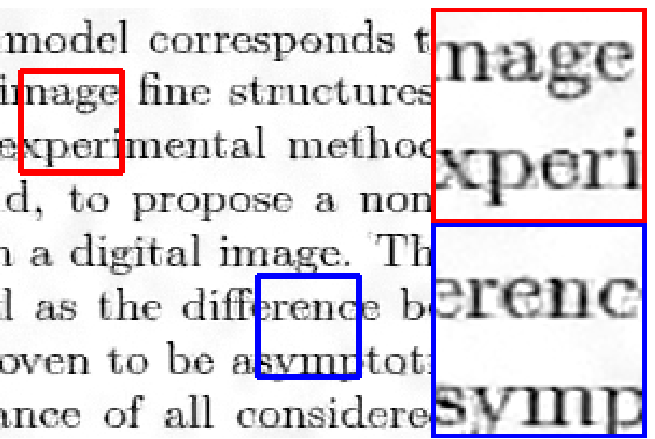}&
\includegraphics[height=\ih,width=\sw]{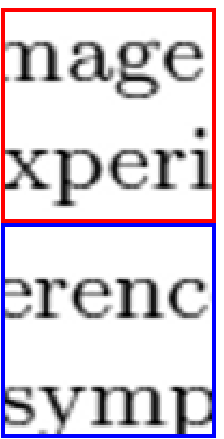}\\
$\sigma = 50$ &  & 20.29 dB (0.8524) & 21.98 dB (0.9311) & 22.49 dB (0.9373) &\\

\includegraphics[height=\ih,width=\ih]{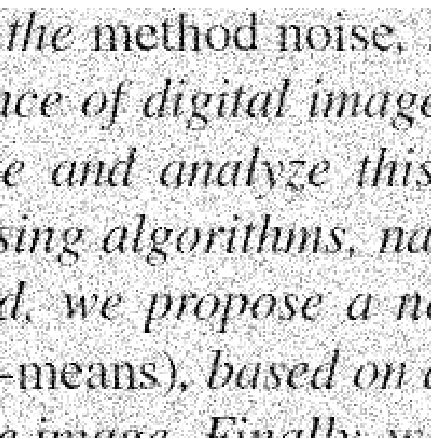}&
\includegraphics[height=\ih,width=\ih]{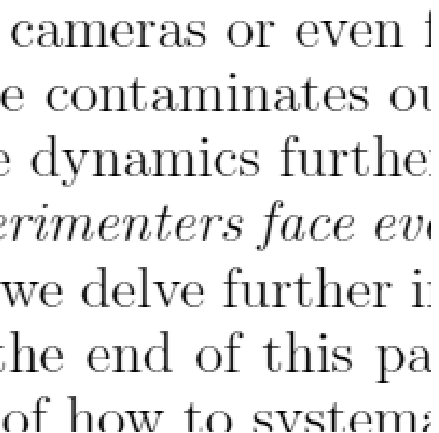}&
\includegraphics[height=\ih,width=\iw]{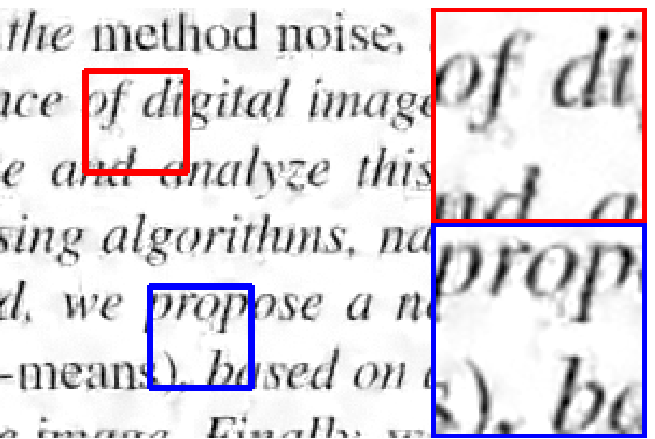}&
\includegraphics[height=\ih,width=\iw]{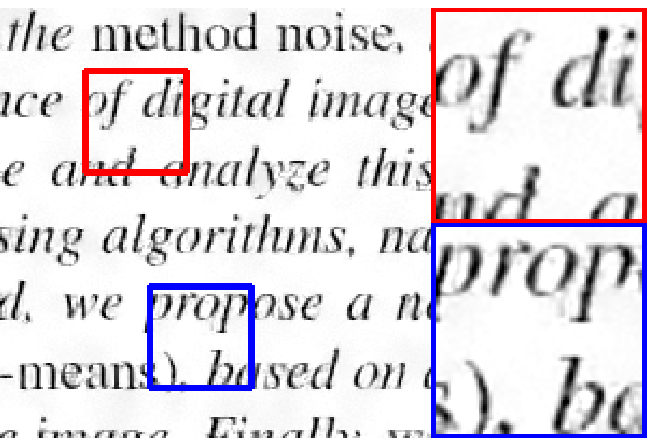}&
\includegraphics[height=\ih,width=\iw]{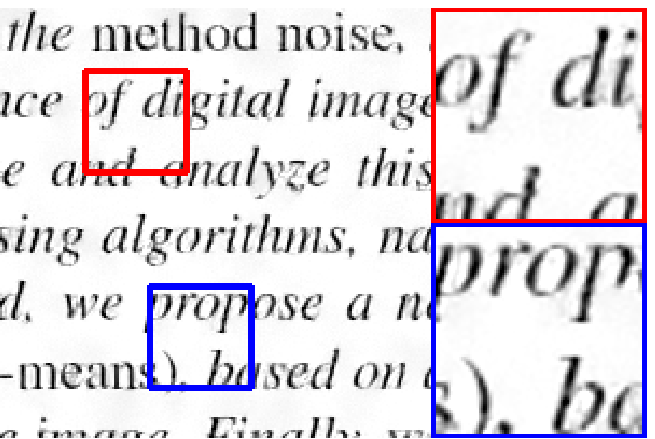}&
\includegraphics[height=\ih,width=\sw]{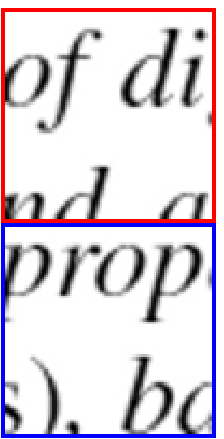}\\
$\sigma = 50$ &  & 21.56 dB (0.8703) & 23.02 dB (0.9302) & 23.50 dB (0.9369) & \\
\end{tabular}
\caption{External image denoising by using an example image for EM adaptation: Visual comparison and objective comparison (PSNR and SSIM in the parenthesis). The flower images are from the 102 flowers dataset \cite{Nilsback_Zisserman_2008}, the face images are from the FEI face dataset \cite{Thomaz_Giraldi_2010}, and the text images are cropped from randomly chosen documents.
\label{fig:external image denoising: visual comparison}}
\end{figure*}

\subsection{Runtime}
Our current implementation is in MATLAB (single thread), and we use an Intel Core i7-3770 CPU with 8 GB RAM. The runtime is about 66 seconds to denoise an image of size $256 \times 256$, where the EM adaptation part takes about 14 seconds, while the MAP denoising part takes about 52 seconds. The EM adaptation utilizes the simplification \eref{eq:simplified covariance adaptation} in Section IV-D, which has significant speedup impact to the adaptation. The MAP denoising part has similar runtime as EPLL, which uses an external mixture model for denoising. As pre-filtering is considered, we note that BM3D takes approximately 0.25 seconds, and EPLL takes approximately 50 seconds.

\section{Conclusion}
We proposed an EM adaptation method to learn effective image priors. The proposed algorithm is rigorously derived from the Bayesian hyper-prior perspective and is further simplified to reduce the computational complexity. In the absence of the latent clean image, we proposed modifications of the algorithm and analyzed how some internal parameters can be automatically estimated. The adapted prior from the EM adaptation better captures the prior distribution of the image of interest and is consistently better than the un-adapted generic one. In the context of image denoising, experimental results demonstrate its superiority over some existing denoising algorithms, such as EPLL and BM3D. Future work includes its extended work on video denoising and other restoration tasks, such as deblurring and inpainting.

\newcommand\numberthis{\addtocounter{equation}{1}\tag{\theequation}}

\appendix

\subsection{Proof of Proposition \ref{proposition:hyper-prior for the reproducing density}}
\label{appendix:proof,hyper-prior for the reproducing density}
\begin{proof}
We first compute the probability that the $i$-th sample belongs to the $k$-th Gaussian component as
\begin{equation}
\gamma_{ki} = \frac{\pi_k^{(m)} \calN(\widetilde{\vp}_i \,|\,\vmu_k^{(m)}, \mSigma_k^{(m)})}{\sum_{l=1}^K \pi_l^{(m)} \calN(\widetilde{\vp}_i \,|\,\vmu_l^{(m)}, \mSigma_l^{(m)})},
\end{equation}
where $\{(\pi_k^{(m)}, \vmu_k^{(m)}, \mSigma_k^{(m)})\}_{k=1}^K$ are the GMM parameters in the $m$-th iteration and let $n_k \bydef\sum_{i=1}^n \gamma_{ki}$. We can then approximate $\log{ f(\widetilde{\vp}_1,\ldots,\widetilde{\vp}_n) | \widetilde{\mTheta})}$ in \eref{hyper-prior: MAP estimation problem} by the Q function as follows
\begin{align}
Q(\widetilde{\mTheta} | \widetilde{\mTheta}^{(m)}) =& \sum_{i = 1}^n \, \sum_{k = 1}^K\; \gamma_{ki} \log{\left( \widetilde{\pi}_k \calN(\widetilde{\vp}_i | \widetilde{\vmu}_k, \widetilde{\mSigma}_k) \right)} \nonumber\\
 \doteq &\sum_{i = 1}^n \, \sum_{k = 1}^K\; \gamma_{ki} \Big( \log{\widetilde{\pi}_k} - \frac{1}{2} \log{| \widetilde{\mSigma}_k|} \nonumber\\
 & - \frac{1}{2} (\widetilde{\vp}_i - \widetilde{\vmu}_k)^T \widetilde{\mSigma}_k^{-1}(\widetilde{\vp}_i - \widetilde{\vmu}_k)\Big) \nonumber\\
 = &\sum_{k = 1}^K n_k ( \log{\widetilde{\pi}_k} - \frac{1}{2}\log{|\widetilde{\mSigma}_k|} ) \nonumber\\
 & - \frac{1}{2} \sum_{k = 1}^K \, \sum_{i = 1}^n\; \gamma_{ki} (\widetilde{\vp}_i - \widetilde{\vmu}_k)^T \widetilde{\mSigma}_k^{-1}(\widetilde{\vp}_i - \widetilde{\vmu}_k),
\label{Q function}
\end{align}
where $\doteq$ indicates that some constant terms that are irrelevant to the parameters $\widetilde{\mTheta}$ are dropped.

We further define two notations
\begin{align*}
\bar{\vmu}_k \bydef \frac{1}{n_k} \sum_{i = 1}^n \gamma_{ki} \widetilde{\vp}_i, \quad \mS_k \bydef \sum_{i = 1}^n\gamma_{ki}(\widetilde{\vp}_i - \bar{\vmu}_k)(\widetilde{\vp}_i - \bar{\vmu}_k)^T.
\end{align*}
Using the equality
\begin{align*}
&\sum_{i = 1}^n \gamma_{ki}(\widetilde{\vp}_i - \widetilde{\vmu}_k)^T \widetilde{\mSigma}_k^{-1}(\widetilde{\vp}_i - \widetilde{\vmu}_k) \\
&= n_k(\widetilde{\vmu}_k - \bar{\vmu}_k)^T \widetilde{\mSigma}_k^{-1}(\widetilde{\vmu}_k - \bar{\vmu}_k) + \mbox{tr}(\mS_k \widetilde{\mSigma}_k^{-1}),
\end{align*}
we can rewrite the Q function as follows
\begin{align*}
Q(\widetilde{\mTheta} | \widetilde{\mTheta}^{(m)}) =& \sum_{k = 1}^K \big\{ n_k (\log{\widetilde{\pi}_k - \frac{1}{2}\log{|\widetilde{\mSigma}_k|}})\\
& \hspace{-6ex} - \frac{n_k}{2}(\widetilde{\vmu}_k - \bar{\vmu}_k)^T \widetilde{\mSigma}_k^{-1}(\widetilde{\vmu}_k - \bar{\vmu}_k) - \frac{1}{2}\mbox{tr}(\mS_k \widetilde{\mSigma}_k^{-1}) \big\}.
\end{align*}
Therefore, we have
\begin{align}
&f(\widetilde{\mTheta} | \widetilde{\vp}_1,\ldots,\widetilde{\vp}_n)  \propto \; \mbox{exp}\big(Q(\widetilde{\mTheta} | \widetilde{\mTheta}^{(m)}) + \log{f(\widetilde{\mTheta})} \big) \nonumber\\
= &\; f(\widetilde{\mTheta}) \prod_{k = 1}^K \Big\{ \widetilde{\pi}_k^{n_k}\, |\widetilde{\mSigma}_k|^{-n_k / 2} \nonumber\\
& \mbox{exp}\big( - \frac{n_k}{2}(\widetilde{\vmu}_k - \bar{\vmu}_k)^T \widetilde{\mSigma}_k^{-1}(\widetilde{\vmu}_k - \bar{\vmu}_k) - \frac{1}{2}\mbox{tr}(\mS_k \widetilde{\mSigma}_k^{-1}) \big) \Big\} \nonumber \\
=& \; \prod_{k = 1}^K \Big\{ \widetilde{\pi}_k^{v_k + n_k - 1} |\widetilde{\mSigma}_k|^{-(\varphi_k + n_k + d + 2)/2} \mbox{exp}\Big( -\frac{\tau_k + n_k}{2} \nonumber\\
&  (\widetilde{\vmu}_k - \frac{\tau_k\vvartheta_k + n_k\bar{\vmu}_k}{\tau_k + n_k})^T \widetilde{\mSigma}_k^{-1}(\widetilde{\vmu}_k - \frac{\tau_k\vvartheta_k + n_k\bar{\vmu}_k}{\tau_k + n_k}) \Big) \nonumber\\
& \mbox{exp} \Big( -\frac{1}{2}\mbox{tr} ( ( \mPsi_k + \mS_k \notag \\
& + \frac{\tau_k n_k}{\tau_k + n_k} (\vvartheta_k - \bar{\vmu}_k)(\vvartheta_k - \bar{\vmu}_k)^T ) \widetilde{\mSigma}_k^{-1} ) \Big)\Big\}.
\end{align}
Defining $v_k' \bydef v_k + n_k,
\varphi_k' \bydef \varphi_k + n_k ,
\tau_k' \bydef \tau_k + n_k ,
\vvartheta_k' \bydef \frac{\tau_k\vvartheta_k + n_k\bar{\vmu}_k}{\tau_k + n_k}$, and $
\mPsi_k' \bydef \mPsi_k + \mS_k + \frac{\tau_k n_k}{\tau_k + n_k} (\vvartheta_k - \bar{\vmu}_k)(\vvartheta_k - \bar{\vmu}_k)^T$, we will get
\begin{equation*}
\begin{array}{l}
f(\widetilde{\mTheta} | \widetilde{\vp}_1,\ldots,\widetilde{\vp}_n) \propto \prod_{k = 1}^K \Big\{\widetilde{\pi}_k^{v_k' - 1}|\widetilde{\mSigma}_k|^{-(\varphi_k' + d + 2)/ 2} \\
\mbox{exp}\big( -\frac{\tau_k'}{2} \left( \widetilde{\vmu}_k - \vvartheta_k' \right)^T  \widetilde{\mSigma}_k^{-1}\left( \widetilde{\vmu}_k - \vvartheta_k' \right) - \frac{1}{2} \mbox{tr}( \mPsi_k' \widetilde{\mSigma}_k^{-1} ) \big) \Big\},
\end{array}
\end{equation*}
which completes the proof.
\end{proof}

\subsection{Proof of Proposition \ref{proposition: solution to optimization for mTheta}}
\label{appendix:proof,solution to optimization for mTheta}
\begin{proof}
We ignore some irrelevant terms and get $\log{f(\widetilde{\mTheta} | \widetilde{\vp}_1,\ldots,\widetilde{\vp}_n)} \doteq \sum_{k = 1}^K \{ (v_k' - 1)\log{\widetilde{\pi}_k} - \frac{(\varphi_k' + d + 2)}{2} \log{|\widetilde{\mSigma}_k|} -\frac{\tau_k'}{2} ( \widetilde{\vmu}_k - \vvartheta_k' )^T  \widetilde{\mSigma}_k^{-1}( \widetilde{\vmu}_k - \vvartheta_k') - \frac{1}{2} \mbox{tr}( \mPsi_k' \widetilde{\mSigma}_k^{-1} ) \}$. Taking derivatives with respect to $\widetilde{\pi}_k$, $\widetilde{\vmu}_k$ and $\widetilde{\mSigma}_k$ will yield the following solutions.
\begin{itemize}
\item Solution to $\widetilde{\pi}_k$.

We form the Lagrangian
\begin{equation*}
J(\widetilde{\pi}_k, \lambda) = \sum_{k = 1}^K (v_k' - 1)\log{\widetilde{\pi}_k} + \lambda \left(\sum_{k = 1}^K \widetilde{\pi}_k - 1 \right),
\end{equation*}
and the optimal solution satisfies
\begin{equation*}
\frac{\partial J}{\partial \widetilde{\pi}_k} = \frac{v_k' - 1}{\widetilde{\pi}_k} + \lambda = 0.
\end{equation*}
It is easy to see that
$\lambda = -\sum_{k = 1}^K (v_k' - 1)$, and thus the solution to $\widetilde{\pi}_k$ is
\begin{align}
\widetilde{\pi}_k =& \frac{v_k' - 1}{\sum_{k = 1}^K (v_k' - 1)} = \frac{(v_k - 1) + n_k}{(\sum_{k = 1}^K v_k - K) + n} \nonumber \\
=& \frac{n}{(\sum_{k = 1}^K v_k - K) + n} \cdot \frac{n_k}{n} \nonumber\\
& + \frac{\sum_{k = 1}^K v_k - K}{(\sum_{k = 1}^K v_k - K) + n} \cdot \frac{v_k - 1}{\sum_{k = 1}^K v_k - K}.
\end{align}

\item Solution to $\widetilde{\vmu}_k$.

We let
\begin{equation}
\frac{\partial L}{\partial \widetilde{\vmu}_k} = -\frac{\tau_k'}{2}\widetilde{\mSigma}_k^{-1}(\widetilde{\vmu}_k - \vvartheta_k') = 0,
\end{equation}
 of which the solution is
\begin{align}
\widetilde{\vmu}_k 
=& \frac{1}{\tau_k + n_k}\sum_{i = 1}^n \gamma_{ki} \widetilde{\vp}_i + \frac{\tau_k}{\tau_k + n_k}\vvartheta_k.
\end{align}

\item Solution to $\widetilde{\mSigma}_k$.

We let
\begin{align}
\frac{\partial L}{\partial \widetilde{\mSigma}_k} =&  -\frac{\varphi_k' + d + 2}{2} \widetilde{\mSigma}_k^{-1} +  \frac{1}{2}\widetilde{\mSigma}_k^{-1} \mPsi_k' \widetilde{\mSigma}_k^{-1} \nonumber\\
&+ \frac{\tau_k'}{2} \widetilde{\mSigma}_k^{-1}( \widetilde{\vmu}_k - \vvartheta_k')( \widetilde{\vmu}_k - \vvartheta_k')^T \widetilde{\mSigma}_k^{-1} = 0, \nonumber
\end{align}
which yields
\begin{equation}
(\varphi_k' + d + 2)\widetilde{\mSigma}_k = \mPsi_k' + \tau_k'( \widetilde{\vmu}_k - \vvartheta_k')( \widetilde{\vmu}_k - \vvartheta_k')^T.
\end{equation}
Thus, the solution is
\begin{align}
\widetilde{\mSigma}_k =& \frac{\mPsi_k' + \tau_k'( \widetilde{\vmu}_k - \vvartheta_k')( \widetilde{\vmu}_k - \vvartheta_k')^T}{\varphi_k' + d + 2} \nonumber\\
=&\; \frac{\mPsi_k + \tau_k (\widetilde{\vmu}_k - \vvartheta_k)(\widetilde{\vmu}_k - \vvartheta_k)^T}{\varphi_k + d + 2 + n_k} \nonumber \\
&\hspace{-5ex} + \frac{n_k (\widetilde{\vmu}_k - \bar{\vmu}_k)(\widetilde{\vmu}_k - \bar{\vmu}_k)^T + \mS_k}{\varphi_k + d + 2 + n_k} \nonumber \\
=&\; \frac{n_k}{\varphi_k + d + 2 + n_k} \frac{1}{n_k} \sum_{i = 1}^n \gamma_{ki}(\widetilde{\vp}_i - \widetilde{\vmu}_k)(\widetilde{\vp}_i - \widetilde{\vmu}_k)^T \nonumber \\
&\hspace{-5ex} + \frac{1}{\varphi_k + d + 2 + n_k}\left( \mPsi_k + \tau_k(\vvartheta_k - \widetilde{\vmu}_k)(\vvartheta_k - \widetilde{\vmu}_k)^T \right). \notag
\end{align}
\end{itemize}
\end{proof}

\section{Appendix}
\subsection{Proof of Proposition \ref{proposition:simplified covariance adaptation}}
\label{appendix:proof,simplified covariance adaptation}
\begin{proof}
We expand the first term in \eref{covariance adaptation}.
\begin{align}
& \alpha_k\frac{1}{n_k} \sum_{i=1}^n\gamma_{ki}(\widetilde{\vp}_i - \widetilde{\vmu}_k)(\widetilde{\vp}_i - \widetilde{\vmu}_k)^T \nonumber\\
= & \;\alpha_k \frac{1}{n_k}\sum_{i=1}^n \gamma_{ki}(\widetilde{\vp}_i\widetilde{\vp}_i^T-\widetilde{\vp}_i\widetilde{\vmu}_k^T - \widetilde{\vmu}_k\widetilde{\vp}_i^T+\widetilde{\vmu}_k\widetilde{\vmu}_k^T) \nonumber\\
\triangleq & \;\alpha_k \frac{1}{n_k}\sum_{i=1}^n \gamma_{ki}\widetilde{\vp}_i\widetilde{\vp}_i^T -(\widetilde{\vmu}_k-(1-\alpha_k)\vmu_k)\widetilde{\vmu}_k^T \nonumber\\
& - \widetilde{\vmu}_k(\widetilde{\vmu}_k - (1-\alpha_k)\vmu_k)^T+\alpha_k\widetilde{\vmu}_k\widetilde{\vmu}_k^T \nonumber\\
= & \;\alpha_k \frac{1}{n_k}\sum_{i=1}^n \gamma_{ki}\widetilde{\vp}_i\widetilde{\vp}_i^T - 2 \widetilde{\vmu}_k\widetilde{\vmu}_k^T \nonumber\\
& + (1-\alpha_k)(\vmu_k\widetilde{\vmu}_k^T+\widetilde{\vmu}_k\vmu_k^T) + \alpha_k\widetilde{\vmu}_k\widetilde{\vmu}_k^T,
\label{term 1 in covariance adaptation}
\end{align}
where $\triangleq$ holds because $\alpha_k\frac{1}{n_k}\sum_{i=1}^n \gamma_{ki}\widetilde{\vp}_i = \widetilde{\vmu}_k - (1-\alpha_k)\vmu_k$ from \eref{mean adaptation}. We then expand the second term in \eref{covariance adaptation}
\begin{align}
&\;(1 - \alpha_k) \left( \mSigma_k + (\vmu_k - \widetilde{\vmu}_k)(\vmu_k - \widetilde{\vmu}_k)^T \right) \nonumber \\
= & \;(1-\alpha_k)(\mSigma_k + \vmu_k\vmu_k^T + \widetilde{\vmu}_k\widetilde{\vmu}_k^T) \notag \\
& \; - (1-\alpha_k)(\vmu_k\widetilde{\vmu}_k^T + \widetilde{\vmu}_k\vmu_k^T). 
\label{term 2 in covariance adaptation}
\end{align}
Combining \eref{term 1 in covariance adaptation} and \eref{term 2 in covariance adaptation} completes the proof.
\end{proof}

\bibliographystyle{IEEEbib}
\bibliography{reference}
\end{document}